\documentclass[prodmode]{acmsmallanon} 

\usepackage[ruled]{algorithm2e}
\renewcommand{\algorithmcfname}{ALGORITHM}
\SetAlFnt{\small}
\SetAlCapFnt{\small}
\SetAlCapNameFnt{\small}
\SetAlCapHSkip{0pt}
\IncMargin{-\parindent}

\usepackage{amsfonts}
\usepackage{amssymb}    
\usepackage{mathrsfs}
\usepackage{amsmath}
\usepackage{graphicx}
\usepackage{tabu}
\usepackage{multicol}
\usepackage{multirow}
\newcommand{\magn}[1]{{\textbf{#1}}}
\newcommand{\refSection}[1]{section~\textbf{\ref{#1}}}
\newcommand{\RefSection}[1]{Section~\textbf{\ref{#1}}}
\newcommand{\refEq}[1]{eq~(\textbf{\ref{#1}})}
\newcommand{\refEqs}[2]{eqs~(\textbf{\ref{#1}}) and (\textbf{\ref{#2}})}
\newcommand{\RefEq}[1]{Eq~(\textbf{\ref{#1}})}
\newcommand{\RefEqs}[2]{Eqs~(\textbf{\ref{#1}}) and (\textbf{\ref{#2}})}
\newcommand{\refFigure}[1]{figure~\textbf{\ref{#1}}}
\newcommand{\RefFigure}[1]{Figure~\textbf{\ref{#1}}}
\newcommand{\refTheorem}[1]{theorem~\textbf{\ref{#1}}}

\newcommand{\refExample}[1]{example~\textbf{\ref{#1}}}

\newcommand{\refClaim}[1]{claim~\textbf{\ref{#1}}}
\newcommand{\RefProposition}[1]{Proposition~\textbf{\ref{#1}}}
\newcommand{\refProposition}[1]{proposition~\textbf{\ref{#1}}}

\newcommand{\RefTable}[1]{Table~\textbf{\ref{#1}}}
\newcommand{\RefAlgorithm}[1]{Algorithm~\textbf{\ref{#1}}}

\newcommand{\refDefinition}[1]{definition~\textbf{\ref{#1}}}
\newcommand{\bullitem}[0]{\ensuremath{\bullet~}}

\newcommand{\powerop}{\ensuremath{{\odot}}}

\newcommand{\ie}[0]{\emph{i.e.},~}
\newcommand{\eg}[0]{\emph{e.g.},~}
\newcommand{\aka}[0]{\emph{a.k.a}.~}

\usepackage[ampersand]{easylist}
\ListProperties(Hide=100, Hang=true, Progressive=3ex, Style*=-- ,
Style2*=$\bullet$ ,Style3*=$\circ$ ,Style4*=\tiny$\blacksquare$ )

\renewcommand{\Re}[0]{\mathbb{R}} 
\newcommand{\Qe}[0]{\mathbb{Q}} 
\newcommand{\Ce}[0]{\mathbb{C}} 
\newcommand{\spsign}[0]{\ensuremath{\frown}} 
\newcommand{\mb}{\ensuremath{\Delta}}

\newcommand{\xx}[0]{\ensuremath{x}}

\newcommand{\zz}[0]{\ensuremath{z}}

\DeclareRobustCommand{\xs}{\ensuremath{\underline{\xx}}}

\newcommand{\zs}[0]{\ensuremath{\underline{\zz}}}

\newcommand{\comment}[1]{}

\newcommand{\opp}[1]{{\overset{#1}{\oplus}}}
\newcommand{\bigopp}[1]{\overset{#1}{\bigoplus}}

\newcommand{\falsemath}{\ensuremath{\text{\textsc{false}}}}
\newcommand{\truemath}{\ensuremath{\text{\textsc{true}}}}

\newcommand{\sn}{\ensuremath{\settype{S}_{N}}}

\newcommand{\func}[1]{\mathsf{#1}}
\newcommand{\jfun}{\ensuremath{\jmath}}
\newcommand{\ff}{\ensuremath{\func{f}}}
\newcommand{\fg}{\ensuremath{\func{g}}}
\newcommand{\pp}{\ensuremath{\func{p}}}

\newcommand{\ph}{\ensuremath{\widehat{\pp}}}
\newcommand{\PP}{\ensuremath{\func{P}}}
\newcommand{\PH}{\ensuremath{\widehat{P}}}

\newcommand{\qq}{\ensuremath{\func{q}}}
\newcommand{\strat}{\ensuremath{\func{s}}}
\newcommand{\QQ}{\ensuremath{\func{Q}}}

\newcommand{\settype}[1]{\ensuremath{\mathcal{#1}}}
\newcommand{\II}{\ensuremath{\mathrm{I}}}
\newcommand{\AAA}{\ensuremath{\mathrm{A}}}
\newcommand{\BBB}{\ensuremath{\mathrm{B}}}
\newcommand{\JJ}{\ensuremath{\mathrm{J}}}
\newcommand{\KK}{\ensuremath{\mathrm{K}}}

\newcommand{\NN}{\ensuremath{\settype{N}}}

\newcommand{\xor}{\ensuremath{{\mathrm{xor}}}}

\newcommand{\OO}{\mathbf{\mathcal{O}}}
\renewcommand{\SS}{\settype{S}}
\newcommand{\FF}{\settype{F}}
\newcommand{\YY}{\settype{Y}}
\newcommand{\XX}{\settype{X}}
\newcommand{\RR}{\settype{Y}^*}
\newcommand{\WW}{\settype{P}}
\newcommand{\sumset}[0]{{ \settype{S}}}
\newcommand{\pset}[0]{{ \settype{D}}}
\newcommand{\semig}{\mathscr{G}}
\newcommand{\semiring}{\mathscr{S}}

\newcommand{\Dn}{\ensuremath{\mathbf{A}}}
\newcommand{\nb}{\partial}

\newcommand{\compclass}{\mho}

\newcommand{\msg}[2]{{\ph}_{#1 \to #2}}
\newcommand{\msgs}{{\underline{\ph}}}
\newcommand{\msgss}[2]{{\underline{\ph}}_{#1 \to #2}}

\newcommand{\msgq}[2]{\func{q}_{#1 \to #2}}
\newcommand{\msgssq}[2]{\underline{\func{q}}_{#1 \to #2}}

\newcommand{\msgt}[2]{{\widetilde{\func{p}}}_{#1 \to #2}}
\newcommand{\pht}{\ensuremath{\widetilde{\pp}}}

\newcommand{\MSGT}[2]{{\widetilde{\func{P}}}_{#1 \to #2}}

\newcommand{\PHT}{\ensuremath{\widetilde{\PP}}}
\newcommand{\MSGBIT}[2]{{\widetilde{{\func{P}}}}_{#1 \leftrightarrow #2}}
\newcommand{\msgbit}[2]{{{\underline{\widetilde{\pp}} }}_{#1 \leftrightarrow #2}}
\newcommand{\MSG}[2]{{\func{P}}_{#1 \to #2}}

\newcommand{\PSP}[0]{\func{S}} 
\newcommand{\msp}[2]{\func{S}_{#1 \to #2}}
\newcommand{\back}[1]{\ensuremath{\setminus #1}}
\newcommand{\ident}{\mathbf{\func{1}}}

\newcommand{\identt}[1]{\overset{#1}{\ident}}
\newcommand{\NP}{\ensuremath{\mathbb{NP}}}
\newcommand{\coNP}{\ensuremath{\mathbf{co}\mathbb{NP}}}
\newcommand{\Aclass}{\ensuremath{\mathbb{A}}}
\newcommand{\PPclass}{\ensuremath{\mathbb{PP}}}
\newcommand{\Poly}{\ensuremath{\mathbb{P}}}
\newcommand{\RP}{\ensuremath{\mathbb{RP}}}
\newcommand{\poly}{\Poly}
\newcommand{\pspace}{\ensuremath{\mathbb{PSPACE}}}
\newcommand{\sharpP}{\ensuremath{\mathbb{\#P}}}
\newcommand{\defeq}{\ensuremath{\overset{\text{def}}{=}}}
\newcommand{\bptimes}{\ensuremath{\overset{\text{}}{\otimes}}}
\newcommand{\bpplus}{\ensuremath{\overset{\text{}}{\oplus}}}
\newcommand{\bigbpplus}{\ensuremath{\overset{\text{}}{\bigoplus}}}
\newcommand{\bigbptimes}{\ensuremath{\overset{\text{}}{\bigotimes}}}
\newcommand{\sptimes}{\ensuremath{\overset{\spsign}{\otimes}}}
\newcommand{\spplus}{\ensuremath{\overset{\spsign}{\oplus}}}
\newcommand{\bigspplus}{\ensuremath{\overset{\spsign}{\bigoplus}}}

\newcommand{\sumop}{\ensuremath{\mathrm{sum}}}
\newcommand{\prodop}{\ensuremath{\mathrm{prod}}}

\newtheorem{claim}[theorem]{Claim}

\begin{document}

\markboth{S. Ravanbakhsh and R. Greiner}{Revisiting Algebra \& Complexity of Inference in Graphical Models}

\title{Revisiting Algebra and Complexity of Inference in Graphical Models}
\author{Siamak Ravanbakhsh
\affil{University of Alberta}
Russell Greiner
\affil{University of Alberta}
}

\begin{abstract}
This paper studies the form and complexity of inference in graphical models using the abstraction offered by algebraic structures.
In particular, we broadly formalize inference problems in graphical models by viewing them as a sequence of operations based on commutative semigroups. We then study the computational complexity
of inference by organizing various problems into
an \textit{inference hierarchy}.
When the underlying structure of an inference problem is a commutative semiring -- \ie a combination of two commutative semigroups with the distributive law --   
a message passing procedure called belief propagation can leverage this distributive law to perform polynomial-time inference for certain problems. 
After establishing the \NP-hardness of inference 
in any commutative semiring, we investigate the relatioin between algebraic properties in this setting and further show that polynomial-time inference using
distributive law does not (trivially) extend to inference problems that are expressed using more than two commutative semigroups.
We then extend the algebraic treatment of message passing procedures to survey propagation, providing a novel perspective using a 
combination of two commutative semirings. This formulation generalizes the application of survey propagation to new settings.
\end{abstract}

\maketitle


\section{Introduction}
Many complicated systems can be modeled as a graphical structure with interacting local functions. Many fields have (almost independently) discovered this:
 grpahical models have been used in
bioinformatics
(protein folding, pedagogy
trees, regulatory networks), 
 neuroscience
(formation of associative memory and
neuroplasticity), 
 communication theory 
(low density parity check codes), 
 statistical physics 
(physics of dense matter and spin-glass theory), 
 image processing
(inpainting, stereo/texture reconstruction, denoising and super-resolution),
 compressed sensing, 
robotics (particle filters), 
 sensor networks, social networks, natural language processing, speech recognition,
combinatorial optimization and 
artificial intelligence
 (artificial neural networks, Bayesian networks).
Two general perspectives have emerged from these varied approaches to local computation, 
namely  variational~\cite{Wainwright2007} versus algebraic~\cite{Aji2000} perspectives on graphical models.
These two perspectives are to some extent ``residuals'' from the different origins of research in AI and statistical physics.

In the statistical study of physical systems, Boltzmann distribution   
 relates the 
probability of each state of a physical system to its energy, 
which is often decomposed due to local interactions \cite{Mezard1987,Mezard09}.
These studies often  
 model systems at the thermodynamic limit of infinite variables and the 
average behaviour through the study of random ensembles.  
Inference techniques with this origin (\eg mean-field and cavity methods) are asymptotically exact under these assumptions. 
Most importantly these studies have reduced inference to optimization through the notion of free energy --\aka the variational approach. 

In contrast, graphical models in the AI community have emerged from the study of knowledge representation
and reasoning under uncertainty \cite{pearl_probabilistic_1988}. These advances are characterized by  their attention to the 
theory of computation and logic \cite{bacchus1991representing}, where  interest in computational (as opposed to analytical) 
solutions has motivated the study of approximability, computational complexity \cite{Cooper90:Computational,roth1993hardness} and 
 the invention of inference techniques such as belief propagation that are efficient and exact on tree structures.  
But, most relevant to the topic of this paper, these studies have lead to algebraic abstractions 
in modeling systems that allow local computation \cite{shenoy1990axioms,lauritzen1997local}.

The common foundation underlying these two approaches is information theory, 
where the derivation of probabilistic principles from logical axioms \cite{jaynes2003probability} leads to
notions such as entropy and divergences that are closely linked to their physical counter-parts
\ie entropy and free energies in  physical systems. At a less abstract level, \cite{Yedidia2001a} showed that
inference techniques in AI and communication are attempting to minimize  (approximations to) free energy (also see \cite{Aji01,Heskes03}).

Another exchange of ideas between the two fields was in the study of critical phenomenon in random constraint satisfaction problems by both computer scientists and physicists \cite{fu_application_1986,mitchell1992hard,monasson1999determining}.
Satisfiability is at the heart of the theory of computation and is used to investigate reasoning in AI. 
On the other hand, the study of critical phenomena and phase transitions 
is central in the statistical physics of disordered systems. This culminated when
a variational analysis lead to the discovery of survey propagation~\cite{mezard_analytic_2002} for constraint satisfaction, which
significantly advanced the state-of-the-art in solving random satisfiability problems.

Despite this convergence, variational and algebraic perspectives are to some extent complementary -- \eg the variational approach does not extend beyond (log) probabilities, while the algebraic approach cannot justify application of message passing to graphs with loops.
This paper is concerned with the algebraic approach. We organize and generalize the previous work on the algebra of graphical models and give several new
results on the complexity and limit of inference in this framework. To this end, 
\refSection{sec:inference} broadly formalizes (and extends) the problem of inference using factor-graphs and commutative semigroups. 
\RefSection{sec:hierarchy} organizes a subset of inference problems into an inference hierarchy 
with increasing levels of computational complexity under \pspace. 
\RefSection{sec:gdl} reviews the distributive property that make efficient inference possible (using Belief Propagation), establishes the difficulty of inference in general commutative semirings and derives a negative result regarding the application of the distributive law beyond commutative semirings.  \RefSection{sec:sp} moves beyond Belief Propagation and introduces an algebraic interpretation of survey propagation that generalizes its application to new settings.

\section{The problem of inference}\label{sec:inference}
We use commutative semigroups to both define what a graphical 
model represents and also to define inference over this graphical model. 
The idea of using structures such as semigroups, monoids and semirings in expressing inference has a long history~\cite{lauritzen1997local,schiex1995valued,bistarelli1999semiring}. 
Our approach, based on factor-graphs~\cite{kschischang_factor_2001} 
and commutative semigroups, generalizes
a variety of previous frameworks, including
Markov networks~\cite{clifford1990markov}, 
Bayesian networks~\cite{pearl1985bayesian}, 
Forney graphs~\cite{forney2001codes}, hybrid models~\cite{dechter2001hybrid}, 
influence diagrams \cite{howard2005influence} and valuation
networks \cite{shenoy1992valuation}. 

In particular, the combination of factor-graphs and
semigroups that we consider here  
generalizes the plausibility, feasibility and utility framework of \cite{pralet2007algebraic},
which is explicitly reduced to  the graphical models mentioned above and many more.
The main difference in our approach is in keeping the framework free of semantics (\eg 
decision and chance variables, utilities, constraints),
that are often associated with variables, factors and operations, without changing the
expressive power.
These notions can later be associated with individual inference problems
 to help with interpretation.

\begin{definition}\label{def:semigroup}
A \magn{commutative semigroup} is a pair $\semig = (\RR, \otimes)$, where $\RR$ is a set and $\otimes: \RR \times \RR \to \RR$ is a binary operation that is 
(I)~associative: $ a \otimes (b \otimes c) = (a \otimes b) \otimes c$ and (II)~commutative: $a \otimes b = b \otimes a$ for all $a, b, c \in \RR$.
A \magn{commutative monoid} is a commutative semigroup plus an identity element $\identt{\otimes}$
such that  $a \otimes \identt{\otimes} = a$. 
If every element $a \in \RR$ has an inverse $a^{-1}$ (often written $\frac{1}{a}$), 
 such that $a \otimes a^{-1} = \identt{\otimes}$,
the commutative monoid is an \magn{Abelian group}.
\end{definition}
Here, the associativity and commutativity properties of a commutative semigroup make the operations invariant to the order of elements.  In general, these properties are not 
``vital'' and one may define inference starting from a \textit{magma}.\footnote{A  magma \cite{pinter2012book} generalizes a semigroup, as it does not require associativity property nor an identity element. Inference in graphical models can be also extended to use magma (in \refDefinition{def:fg}). For this, 
the elements of $\RR$ and/or $\XX$ should be ordered and/or parenthesized so as to avoid ambiguity in the order of pairwise operations over the set. Here, to avoid unnecessary complications, we confine our treatment to commutative semigroups.}

\begin{example} Some examples of semigroups are:
\begin{easylist}
& The set of strings with the concatenation operation forms a semigroup with the 
empty string as the identity element. However this semigroup is not commutative.
& The set of natural numbers $\settype{N}$ with summation defines a commutative semigroup. 
& Integers modulo $n$ with addition defines an Abelian group.
& The power-set $2^{\SS}$ of any set $\SS$, with intersection operation defines a commutative semigroup with $\SS$ as its identity element.
& The set of natural numbers with greatest common divisor defines a commutative monoid with $0$ as its identity. In fact any semilattice is a commutative semigroup~\cite{davey2002introduction}.
& Given two commutative semigroups on two sets $\RR$ and $\settype{Z}^*$, their Cartesian product is also a commutative semigroup.
\end{easylist}
\end{example}

Let $\xs = (\xx_1,\ldots,\xx_N)$ be a tuple of $N$ discrete
variables $\xx_i\in \XX_i$, where $\XX_i$ is the finite domain of
$\xx_i$  and $\xs \in \XX =  \XX_1 \times \ldots \times \XX_N$.  
Let $\II \subseteq \NN = \{1,2,\ldots,N\}$ denote a subset
of variable indices and $\xs_\II \!=\! \{ \xx_i\! \mid\! i\in \II\} \in \XX_\II$ be the
tuple of variables in $\xs$ indexed by the subset $\II$.
A factor $\ff_{\II}: \XX_{\II} \to \YY_\II$ is a function over a subset of variables and 
$\YY_\II = \{\, \ff_\II(\xs_\II) \, \mid\, \xs_\II \in \XX_\II\, \}$ is the range of this factor.

\begin{definition}\label{def:fg}
 A \magn{factor-graph}  is a pair $(\FF, \semig)$  such that
\begin{itemize}
\item $\FF = \{ \ff_{\II}\}$ is a collection of factors with collective range  $\YY = \bigcup_\II \YY_\II$.
\item  $|\FF| = \mathrm{Poly}(N)$.  
\item $\ff_\II$ has a polynomial representation in $N$ and it is possible to evaluate $\ff_\II(\xs_\II)\; \forall \II, \xs_\II$ in polynomial time.
\item $\semig = (\RR, \otimes)$ is a commutative semigroup, where  $\RR$ is the closure of $\YY$ w.r.t. $\otimes$.
\end{itemize}
The factor-graph compactly represents the expanded (joint) form
\begin{align}\label{eq:expanded}
  \qq(\xs) = \bigotimes_\II \ff_\II(\xs_\II)
\end{align}
\end{definition}


Note that the connection between the set of factors $\FF$ and the commutative semigroup is through the ``range''  of factors.
The conditions of this definition are necessary and sufficient to 1) compactly represent a factor-graph and 2) evaluate the expanded form, $\qq(\xs)$, in polynomial time.
In the following we make a stronger assumption that ensures a factor has a tractable tabular form -- that is $|\XX_\II| = \mathrm{Poly}(N)$, which means  $\ff_\II(\xs_\II)$ 
can be explicitly expressed for each $\xs_\II \in \XX_\II$ as 
an element of $|\II|$-dimensional array or table.\footnote{Important factor-graphs that violate this assumption are the ones with high-order sparse factors \cite{Tarlow2010,Potetz2008}. 
Although it is possible to obtain polynomial time message passing updates for special high-order factors, general high-order factors do not have
polynomial representation.}
 



$\FF$ can be conveniently represented as a bipartite graph 
that includes two sets of nodes: variable nodes $\xx_i$, and factor 
 nodes ${\II}$. A variable node $i$ (note that we will often
identify a variable $\xx_i$ with its index ``$i$'') is connected to a
factor node $\II$ if and only if $i \in \II$ --\ie $\II$ is a set that is also an index. 
We will use $\nb$ to denote the neighbours of a variable or factor node in the factor
graph -- that is $\nb \II = \{\, i \; \mid \; i \in \II\, \}$ (which is
the set $\II$) and $\nb i = \{\, \II \; \mid \; i \in \II\, \}$.

Also, we use $\mb i$ to denote the \magn{Markov blanket} of node $\xx_i$ -- \ie $\mb i 
= \{ j \in \nb \II\; \mid \; \II \in \nb i,\; j \neq i\}$.
\begin{example}\label{example:fig1}
\refFigure{fig:fg} shows a factor-graph with 12 variables and 12 factors.
Here \\ $\xs = (\xx_i, \xx_j,\xx_k, \xx_e, \xx_m, \xx_o, \xx_r, \xx_s, \xx_t, \xx_u, \xx_v, \xx_w)$, $\II = \nb \II =  \{i,j,k\}$,  $\xs_\mathrm{K} = \xs_{\{k,w,v\}}$
and $\nb j = \{\II, \mathrm{V}, \mathrm{W}\}$. 
Assuming $\semig_e = (\Re, \min)$, the expanded
form represents  $$\qq(\xs) = \min \{\, \ff_{\II}(\xs_\II), \ff_{\JJ}(\xs_\JJ),\ldots,\ff_{\mathrm{Z}}(\xs_{\mathrm{Z}})\, \}.$$

Now, assume that all variables are binary -- \ie $\XX = \{0,1\}^{12}$ and $\qq(\xs)$ is 12-dimensional hypercube, with one value per each assignment at each corner. Also assume
each of the factors count the number of non-zero variables -- \eg for $\zs_{\mathrm{W}} = (1,0,1) \in \XX_{\mathrm{W}}$ we have $\ff_{\mathrm{W}}(\zs_{\mathrm{W}}) = 2$. Then, for the complete assignment 
$\zs = (0,1,0,1,0,1,0,1,0,1,0,1) \in \XX$, it is easy to check that the expanded form is 
$\qq(\zs) = \min \{2, 0, 1,\ldots, 1\} = 0$. 
\end{example}

A marginalization operation shrinks the expanded form $\qq(\xs)$ using another commutative semigroup with binary operation $\oplus$.
Inference is a combination of an expansion and one or more marginalization operations,
which can be computationally intractable due to the exponential size of the expanded form.
\begin{definition} 
Given a function $\qq: \XX_{\JJ} \to \YY$, and a commutative semigroup $\semig = (\RR, \oplus)$, where $\RR$ is the closure of $\YY$ w.r.t. $\oplus$,
 the marginal of $\qq$ for $\II \subset \JJ$ is 
\begin{align}\label{eq:marginalization}
 \qq(\xs_{\JJ \back \II}) \quad \defeq \quad \bigoplus_{\xs_{\II}} \qq(\xs_{\JJ}) 
\end{align}
where $\bigoplus_{\xs_{\II}} \qq(\xs_{\JJ})$ is short for $\bigoplus_{\xs_{\II} \in \XX_{\II}} \qq(\xs_{\JJ \back \II}, \xs_{\II})$, 
which means to compute $\qq(\xs_{\JJ \back \II})$ for each $\xs_{\JJ \back \II}$, one should perform the operation $\oplus$
over the set of all the assignments to the tuple $\xs_{\II} \in \XX_{\II}$.
\end{definition}
We can think of $\qq(\xs_\JJ)$ as a $\vert \JJ \vert$-dimensional tensor and marginalization as performing $\oplus$ operation over the axes
in the set $\II$. The result is another $\vert \JJ \back \II \vert$-dimensional tensor (or function) that we call the \magn{marginal}. 
Here if the marginalization is over all the dimensions in $\JJ$, we denote the marginal by $\qq(\emptyset)$ instead of $\qq(\xs_\emptyset)$ and call it the \magn{integral} of $\qq$.

\begin{figure}
\centering
\includegraphics[width=.5\textwidth]{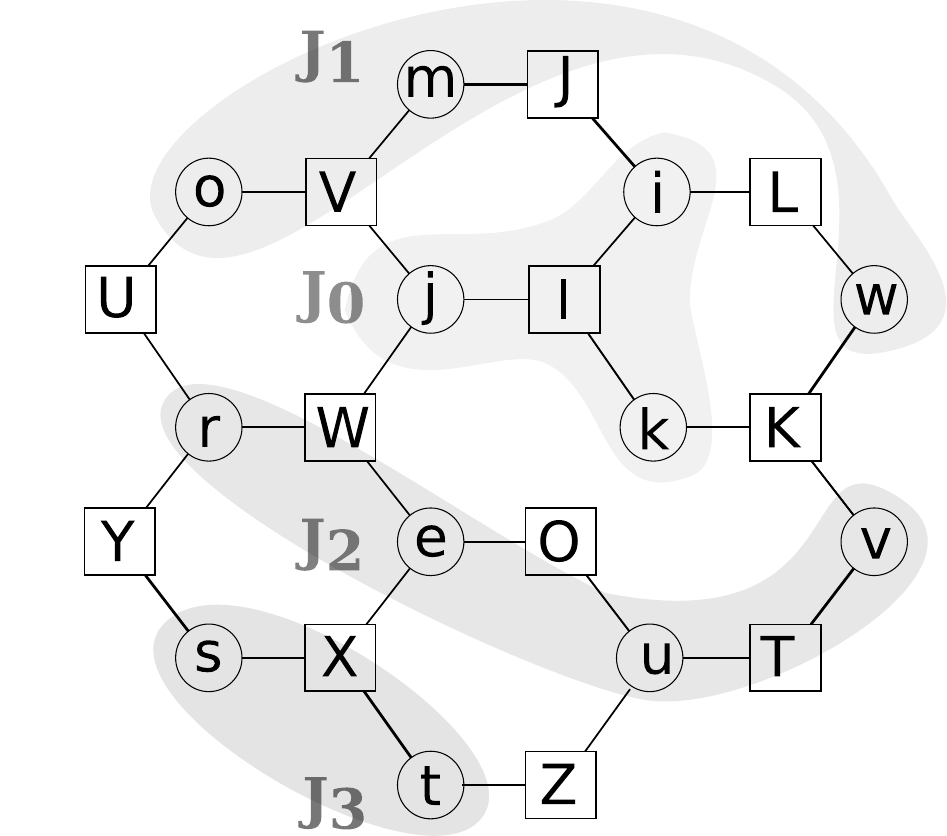}
\caption{A factor-graph with variables as circles and factors as squares. The partitioning of variables indicated by the shading in parts defines an inference problem over this factor-graph.}
\label{fig:fg}
\end{figure}

Now we define an inference problem as a sequence of marginalizations over the expanded form of a factor-graph.
\begin{definition}\label{def:inference}
An \magn{inference problem} seeks
  \begin{align}
    \qq(\xs_{\JJ_0}) = \bigopp{M}_{\xs_{\JJ_M}} \bigopp{M-1}_{\xs_{\JJ_{M-1}}} \ldots \bigopp{1}_{\xs_{\JJ_1}} \bigotimes_{\II} \ff_\II(\xs_\II) 
  \end{align}
where
  \begin{itemize}
  \item $\RR$ is the closure of $\YY$ (the collective range of factors), w.r.t. $\opp{1},\ldots,\opp{M}$ and $\otimes$. 
  \item $\semig_m = (\RR, \opp{m})\quad \forall 1 \leq m \leq M$ and $\semig_e=(\RR, \otimes)$ are all  commutative semigroups.
\item $\JJ_0,\ldots,\JJ_L$ partition the set of variable indices $\NN = \{1,\ldots,N\}$. 
\item $\qq(\xs_{\JJ_0})$ has a polynomial representation in $N$ -- \ie $|\XX_{\JJ_0}| = \mathrm{Poly}(N)$ 
\end{itemize}
\end{definition}

Note that $\opp{1},\ldots,\opp{M}$  refer to potentially different operations as each belongs to a different semigroup.
When $\JJ_0 = \emptyset$, we call the inference problem \magn{integration} (denoting the inquiry by $\qq(\emptyset)$) and otherwise we call it \magn{marginalization}.
Here, having a constant sized $\JJ_0$ is not always enough to ensure that $\qq(\xs_{\JJ_0})$  has a polynomial representation in $N$. This is because 
the size of $\qq(\xs_{\JJ_0})$ for any individual $\xs_{\JJ_0} \in \XX_{\JJ_0}$ may grow exponentially with $N$ (\eg see \refClaim{th:product_pspace}).
In the following we call $\semig_e = (\RR, \otimes)$ the expansion semigroup and $\semig_m = (\RR, \opp{m})$ for each $m$ in $\{1,\ldots,M\}$
is a marginalization semigroup.


\begin{example}\label{example:fig2}
Going back to \refExample{example:fig2}, the shaded region in \refFigure{fig:fg}
shows a partitioning of the variables that we use to define the following inference problem:
$$
\qq(\xs_{\JJ_0}) = \max_{\xs_{\JJ_3}} \sum_{\xs_{\JJ_2}} \min_{\xs_{\JJ_1}} \min_\II \ff_\II(\xs_\II) 
$$
We can associate this problem with the following semantics: we may 
think of each factor as an agent, where
$\ff_{\II}(\xs_\II)$ is the payoff for agent $\II$, which only depends
on a subset of variables $\xs_\II$. 
We have adversarial variables ($\xs_{\JJ_1}$), environmental or chance variables ($\xs_{\JJ_2}$), controlled variables ($\xs_{\JJ_3}$) and query variables ($\xs_{\JJ_0}$).  
For the inference problem above, 
each query $\xs_{\JJ_0}$ seeks to maximize 
the expected minimum payoff of all agents,
without observing the adversarial or chance variables, and 
assuming the adversary makes its decision after observing control and chance variables.
\end{example}

\begin{example}\label{example:ising}
A ``probabilistic'' graphical model is defined using a expansion semigroup $\semig_e = (\Re^{\geq 0}, \times)$ and often a marginalization semigroup $\semig_m = (\Re^{\geq 0}, +)$. The expanded form represents the unnormalized joint probability
$\qq(\xs) = \prod_{\II}\ff_{\II}(\xs_{\II})$, whose marginal probabilities are simply called marginals.
Replacing the summation with the marginalization semigroup $\semig_m = (\Re^{\geq 0}, \max)$, seeks the
 maximum probability state and the resulting integration problem 
$\qq(\emptyset) =  \max_{\xs} \prod_{\II} \ff_\II(\xs_\II)$
is known as \magn{maximum a posteriori (MAP)} inference. 
\index{MAP}
\index{maximum a posteriori|see{MAP}}
\index{max-product}
 Alternatively by adding a second marginalization operation to the summation, we get the \magn{marginal MAP} inference
\index{marginal MAP}
\index{max-sum-product}
\begin{align}\label{eq:marginalmap}
    \qq(\xs_{\JJ_0}) \quad = \quad \max_{\xs_{\JJ_2}} \sum_{\xs_{\JJ_1}} \prod_{\II} \ff_{\II}(\xs_\II).
\end{align}
where here $\bigotimes = \prod$, $\bigopp{1} = \sum$ and $\bigopp{2} = \max$ (recall the numbering of operations is left to right.)

If the object of interest is the negative 
\index{energy}
log-probability (\aka energy), the product expansion semigroup is
replaced by $\semig_e=(\Re, +)$. Instead of the sum marginalization semigroup, we can
\index{log-sum-exp semigroup}
\index{semigroup!log-sum-exp}
use the \magn{log-sum-exp} semigroup, $\semig_m = (\Re, +)$ where 
$a \oplus b \defeq \log(e^{-a} + e^{-b})$. The integral in this case is the log-partition function.
If we change the marginalization semigroup to $\semig_m = (\Re, \min)$, the integral is the minimum energy (corresponding to MAP).
\end{example}

\section{The inference hierarchy}\label{sec:hierarchy}
Often, the complexity class is concerned with the \magn{decision version} of the 
inference problem in \refDefinition{def:inference}. The decision version of an inference problem asks a yes/no question about the integral such as $\qq(\emptyset) \overset{?}{\geq} q$ \footnote{Assuming an ordering on $\RR$.} or 
$\qq(\emptyset) \overset{?}{=} q$
for a given $q \in \RR$. 



Here, we produce a hierarchy of inference problems  
in analogy to the polynomial~\cite{stockmeyer1976polynomial}, the counting~\cite{wagner1986complexity} and  the arithmetic~\cite{rogers1987theory} hierarchies. 

To define the hierarchy, we 
assume the following in \refDefinition{def:inference}:
\begin{easylist}
& Any two consecutive marginalization operations are distinct ($\opp{l} \neq \opp{{l+1}} \; \forall 1 \leq l < M$).
& The marginalization index sets $\JJ_l \; \forall 1 \leq l \leq M$ are non-empty. Moreover if 
$| \JJ_l| = \OO(\log(N))$ we call this marginalization operation a \magn{polynomial marginalization} as here $| \XX_{\JJ_{l}} | = \mathrm{Poly}(N)$. 
& In defining the factor-graph, we required each factor to be polynomially computable.
In building the hierarchy, we require the operations over each semigroup to be polynomially computable as well. To this end we consider 
the set of rational numbers $\RR \subseteq \Qe^{\geq 0} \cup \{\pm \infty\}$. Note that this automatically eliminates semigroups that involve operations such as exponentiation and logarithm (because $\Qe$ is not closed under these operations) and only consider summation, multiplication, minimization
and maximization. 
\end{easylist}
We can always re-express any inference problem to enforce the first two conditions and therefore they do not impose any restriction. Below, we will use a \magn{language} to identify inference problems for an arbitrary set of factors $\FF = \{\ff_\II\}$. For example, sum-product refers to the inference problem $\sum_{\xs} \prod_\II \ff_\II(\xs_\II) \overset{?}{\geq} q$.
In this sense the rightmost ``token'' in the language (here \textit{product}) identifies the expansion semigroup $\semig_e = (\mathbb{Q}, \times)$ and the rest of tokens identify the marginalization semigroups over $\mathbb{Q}$ in the given order. Therefore, this minimal language exactly identifies the inference problem. The only information that affects the computational complexity of an inference
problem but is not specified in this language is whether each of the marginalization operations are polynomial or exponential.

We define five \magn{inference families}: $\Sigma, \Pi, \Phi, \Psi, \Delta$. The families
are associated with that ``outermost'' marginalization operation  -- \ie $\opp{M}$ in \refDefinition{def:inference}. 
$\Sigma$ is the family of
inference problems where $\opp{M} = \sumop$. Similarly,
$\Pi$ is associated with product, $\Phi$ with minimization and $\Psi$ with maximization.
$\Delta$ is the family of inference problems where the last marginalization is polynomial (\ie $|\JJ_M| = \OO(\log(N))$ regardless of $\opp{M}$).

Now we define \magn{inference classes} in each family, such that all the problems in the same class
have the same computational complexity. Here, the hierarchy is exhaustive -- \ie
 it includes all inference problems using any combination of the four operations sum, min, max and product
 whenever the integral $\qq(\emptyset)$ has a polynomial representation (see \refClaim{th:product_pspace}). 
 Moreover the inference classes are disjoint.
 For this, each family
is parameterized by a subscript $M$ and two sets $\sumset$ and $\pset$ (\eg $\Phi_M(\sumset, \pset)$ is an inference ``class'' in family $\Phi$). As before, $M$
is the number of marginalization operations, $\sumset$ is the set of indices
of the (exponential) $\sumop$-marginalization and $\pset$ is the set of indices of
polynomial marginalizations.
\begin{example}
Sum-min-sum-product identifies the decision problem
\begin{align*}
\sum_{\xs_{\JJ_3}} \min_{\xs_{\JJ_2}} \sum_{\xs_{\JJ_1}} \prod_{\II} \ff_\II(\xs_\II) \quad \overset{?}{\geq} \quad q
\end{align*}
where $\JJ_1$, $\JJ_2$ and $\JJ_3$ partition $\settype{N}$. Assume
$\JJ_1 = \{2,\ldots,\frac{N}{2}\}$, $\JJ_2 = \{\frac{N}{2} + 1,\ldots, N\}$ and $\JJ_3 = \{1\}$. Since we have three marginalization operations $M = 3$.
Here the first and second marginalizations are exponential and the third one is polynomial (since $|\JJ_3|$ is constant). Therefore $\pset = \{3\}$.
Since the only exponential summation is $\bigopp{1}_{\xs_{\JJ_1}} = \sum_{\xs_{\JJ_1}}$,
$\sumset = \{1\}$.
In our inference hierarchy, this problem belongs to the class $\Delta_{3}( \{1\}, \{3\})$. 

Alternatively, if we use different values for $\JJ_1$, $\JJ_2$ and $\JJ_3$ whose sizes each grow linearly with $N$,
the corresponding inference problem becomes a member of  $\Sigma_3(\{1,3\}, \emptyset)$.
\end{example}
\begin{remark}
Note that arbitrary assignments to $M$, $\sumset$ and $\pset$ do not necessarily define a valid inference class.
For example we require that $\sumset \cap \pset = \emptyset$ and no index in $\pset$ or $\sumset$
can be larger than $M$. Moreover, the values in $\sumset$ and $\pset$ should be compatible with the inference class. For example, for inference class $\Sigma_M(\sumset,\pset)$, $M$ is a member of $\sumset$.
For notational convenience, if an inference class notation is invalid we equate it with an empty set -- \eg $\Psi_1(\{1\}, \emptyset) = \emptyset$, because $\sumset = \{1\}$ and $M=1$ means the inference class is $\Sigma$ rather than $\Psi$. 
\end{remark}

In the definition below, we ignore the inference problems in which product
appears in any of the marginalization semigroups (\eg product-sum).
The following claim, explains this choice. 
\begin{claim}\label{th:product_pspace}
For $\oplus_M = \prodop$, the inference query $\qq(\xs_{\JJ_0})$ can have an exponential representation in $N$. 
\end{claim}
\begin{proof}
The claim states that when the product appears in the marginalization operations, the marginal (and integral) can become very large, such that we can no longer represent them in space polynomial in $N$.
 We show this for an integration problem. The same idea can show the exponential representation of a marginal query. 

To see why this integral has an exponential representation in $N$, consider its simplified form
\begin{align*}
  \qq(\emptyset) = \prod_{\xs_{\II}} \qq(\xs_{\II})
\end{align*}
where $\qq(\xs)$ here is the result of inference up to the last marginalization step $\opp{M}$, which is product, where $\XX_\II$ grows exponentially with $N$. 
Recall that the hierarchy is defined for operations on $\mathbb{Q}^{\geq 0}$. Since $\qq(\xs_\II)$ for each $\xs_\II \in \XX_\II$ 
has a constant size, say  $c$, the size of representation of $\qq(\emptyset)$ using a binary scheme is 
\begin{align*}
  \left \lceil \log_2(\qq(\emptyset)) \right \rceil = \left \lceil \log_{2} \big ( \prod_{\xs_{\II}} \qq(\xs_{\II}) \big ) \right \rceil = \left \lceil \sum_{\xs_{\II}} c \right \rceil = \left \lceil c \vert \XX_\II \vert \right \rceil
\end{align*}
which is exponential in $N$.
\end{proof}

Define the \magn{base members} of the families as
\begin{align} \label{eq:basemembers}
\Sigma_0(\emptyset, \emptyset) \defeq \{\sumop\} \quad &
\Phi_0(\emptyset, \emptyset) \defeq \{ \min \} \\
\Psi_0(\emptyset, \emptyset) \defeq \{ \max \} \quad &
\Pi_0(\emptyset, \emptyset) \defeq \{ \prodop \} \notag \\
\Delta_0(\emptyset, \emptyset) = \emptyset \quad &\Delta_{1}(\emptyset, \{1\}) \defeq \{\sumop-\sumop,\, \min-\min,\, \max-\max \} \notag
\end{align}
where the initial members of each family only identify the expansion semigroup -- \eg $\sumop$ in $\Sigma_0(\emptyset, \emptyset)$ identifies $\qq(\xs) = \sum_{\II} \ff_{\II}(\xs_\II)$.  
Here, the exception is $\Delta_1(\emptyset, \{1\})$, which contains three \textit{inference problems}.\footnote{We treat $M=1$ for $\Delta$ specially as in this case the marginalization operation can not be polynomial. This is because if $|\JJ_1| = \OO(\log(N))$, then $|\JJ_0| = \Omega(N)$ which violates the conditions in the definition of the inference problem. 
}

Let $\Xi_M(\sumset, \pset)$ denote the union of corresponding classes within all families:
\begin{align*}
  \Xi_M(\sumset, \pset) = \Sigma_M(\sumset, \pset) \cup \Pi_M(\sumset, \pset) \cup \Phi_M(\sumset, \pset) \cup \Psi_M(\sumset, \pset) \cup \Delta_M(\sumset, \pset)
  \end{align*}
Now define the \magn{inference family members} recursively, by adding a marginalization operation to all the problems in each inference class. If this marginalization is polynomial then the new class belongs to the $\Delta$ family and the set $\pset$ is updated accordingly. Alternatively, if this outermost marginalization is exponential, depending on the new marginal operation (\ie $\min, \max, \sumop$) the new class is defined to be a member of $\Phi, \Psi$ or $\Sigma$. For the case that the last marginalization is summation set $\sumset$ is updated.

\noindent {\bullitem\ \textbf{Adding an exponential marginalization}} $ \forall \; |\XX_{\JJ_M}| = \mathrm{Poly}(N),\, M > 0$
\begin{align}
\Sigma_{M+1}(\sumset \cup \{M+1\}, \pset) &\defeq \big \{ \sumop - \xi \mid \xi \in  \Xi_M(\sumset, \pset) \back \Sigma_{M}(\sumset, \pset)  \} \label{eq:hierarchymembers1}\\
\Phi_{M+1}(\sumset, \pset) &\defeq \big \{ \min - \xi \mid \xi \in \Xi_M(\sumset, \pset) \back \Phi_{M}(\sumset, \pset) \big \} \notag \\
\Psi_{M+1}(\sumset, \pset) &\defeq \big \{ \max - \xi \mid \xi \in \Xi_M(\sumset, \pset) \back \Psi_{M}(\sumset, \pset) \big \} \notag \\
\Pi_{M+1}(\sumset, \pset) &\defeq  \emptyset \notag 
\end{align}

\noindent \bullitem\ \textbf{Adding a polynomial marginalization} $\quad \forall \; |\XX_{\JJ_M}| =  \mathrm{Poly}(N),\, > 1$  
  \begin{align}\label{eq:hierarchymembers2}
\Delta_{M+1}(\sumset, \pset \cup \{M+1\}) &\defeq \big \{ \oplus-\xi \mid 
 \xi \in  \Xi_M(\sumset, \pset) \; , \oplus \in \{\min,\max,\sumop \} \big \}
\end{align}


\subsection{Single marginalization}
Following the recursive construction of the hierarchy as defined above,
the inference classes in the hierarchy with \textit{one} marginalization are 
\begin{align}
\Delta_1(\emptyset, \{1\}) &= \{\min-\min,\; \max-\max,\; \sumop-\sumop\}\\
\Psi_1(\emptyset, \emptyset)  &= \{\max-\min,\; \max-\sumop,\; \max-\prodop\} \\
\Phi_1(\emptyset, \emptyset)  &= \{\min-\max,\; \min-\sumop,\; \min-\prodop\} \\
\Sigma_1(\{1\}, \emptyset)  &=  \{ \sumop-\prodop,\; \sumop-\min,\; \sumop-\max \}
\end{align}

Now we review all the problems above and prove that $\Delta_1, \Psi_1, \Phi_1$ and $\Sigma_1$ are complete w.r.t. \poly, \NP, \coNP\ and \PPclass\ respectively.  
Starting from $\Delta_1$:
\begin{proposition}\label{th:sumsum_poly}
sum-sum, min-min and max-max inference are in \poly.
\end{proposition}
\begin{proof}
To show that these inference problems are in \poly, we 
provide polynomial-time algorithms for them:

\noindent \bullitem\ $\sumop-\sumop$ is short for
\begin{align*}
  \qq(\emptyset) = \sum_{\xs} \sum_{\II} \ff_\II(\xs_\II)
\end{align*}
which asks for the sum over all assignments of $\xs \in \XX$, of the sum
of all the factors. It is easy to see that each factor value $\ff_\II(\xs_\II)\; \forall \II,\ \XX_\II$
is counted $\vert \XX_{\back \II} \vert$ times in the summation above.
Therefore we can rewrite the integral above as
\begin{align*}
  \qq(\emptyset) = \sum_{\II} \vert \XX_{\back \II} \vert \big ( \sum_{\xs_{\II}} \ff_{\II}(\xs_\II) \big )
\end{align*}
where the new form involves polynomial number of terms and therefore is easy to calculate.

\noindent \bullitem\ $\min-\min$ (similar for $\max-\max$) is short for
\begin{align*}
  \qq(\emptyset) = \min_{\xs} \min_{\II} \ff_\II(\xs_\II)
\end{align*}
where the query seeks the minimum achievable value of any factor.
We can easily obtain this by seeking the range of all factors and reporting the minimum value in polynomial time. 
\end{proof}

Max-sum and max-prod are widely studied and it is known that their decision version are \NP-complete~\cite{shimony1994finding}.
By reduction from satisfiability we can show that max-min inference~\cite{ravanbakhsh_minmax} is also \NP-hard.

\begin{proposition}\label{th:minmaxnp}
The decision version of max-min inference that asks $\max_{\xs} \min_{\II} \ff_{\II}(\xs_\II) \overset{?}{\geq} q$ is \NP-complete.
\end{proposition}
\begin{proof}
Given $\xs$ it is easy to verify the decision problem, so max-min decision belongs to \NP. 
To show \NP-completeness, we reduce the 3-SAT to a max-min inference problem, such that
3-SAT is satisfiable \textit{iff} the max-min value is $\qq(\emptyset) \geq 1$ and unsatisfiable otherwise. 

Simply define one factor per clause of 3-SAT, such that $\ff_{\II}(\xs_{\II}) = 1$
if $\xs_{\II}$ satisfies the clause and any number less than one otherwise.
With this construction, the max-min value $\max_{\xs} \min_{\II \in \FF} \ff_{\II}(\xs_\II)$ is one \emph{iff} the original SAT problem was satisfiable, otherwise it is less than one. This reduces 3-SAT to Max-Min-decision.
\end{proof}

This means all the problems in $\Psi_1(\emptyset, \emptyset)$ are in \NP\ (and in fact are complete w.r.t.~this complexity class).
In contrast, problems in $\Phi_1(\emptyset, \emptyset)$ are in \coNP, which is the class of decision problems in which the ``NO instances'' result has a polynomial time verifiable witness or proof.
Note that by changing the decision problem from $\qq(\emptyset) \overset{?}{\geq} q$ to $\qq(\emptyset) \overset{?}{\leq} q$, the complexity classes
of problems in $\Phi$ and $\Psi$ family are reversed (\ie problems in $\Phi_1(\emptyset, \emptyset)$ become $\NP$-complete and the problems in $\Psi_1(\emptyset, \emptyset)$ become \coNP-complete).

Among the members of $\Sigma_1(\{1\}, \emptyset)$, sum-product is known to be \PPclass-complete~\cite{littman2001stochastic,roth1993hardness}.
It is easy to show the same result for sum-min (sum-max) inference.
\begin{proposition}\label{th:summinsharpp}
The sum-min decision problem $\sum_{\xs} \min_{\II} \ff_\II(\xs_\II) \overset{?}{\geq} q$ is \PPclass-complete for $\YY = \{0,1\}$.
\end{proposition}
 \PPclass\ is the class of problems that are polynomially solvable using a non-deterministic Turing machine, where the acceptance condition is that the majority of computation paths accept.
\begin{proof}
To see that $\sum_{\xs} \min_\II \ff_{\II}(\xs_\II) \overset{?}{\geq} q$ is in \PPclass, 
 enumerate all $\xs \in \XX$ non-deterministically and for each assignment
calculate $\min_\II \ff_{\II}(\xs_\II)$ in polynomial time 
(where each path accepts iff $\min_\II \ff_{\II}(\xs_\II) = 1$) and accept
iff at least $q$ of the paths accept.

Given a matrix $\Dn \in \{0,1\}^{N \times N}$ the problem of calculating its permanent
\begin{align*}
  \mathsf{perm}(\Dn) = \sum_{\zs \in \sn} \prod_{i = 1}^{N} \Dn_{i, \zz_i}
\end{align*}
where $\sn$ is the set of permutations of $1,\ldots,N$ 
is \sharpP-complete and the corresponding decision problem is \PPclass-complete~\cite{valiant1979complexity}.
To show completeness w.r.t. \PPclass\ it is enough to reduce the problem of computing the matrix permanent to sum-min inference in a graphical model.

The problem of computing the permanent has been reduced to sum-product inference in graphical models  \cite{huang2009approximating}. However, when $\ff_{\II}(\xs_\II) \in \{0,1\}\;\forall \II$,
sum-product is isomorphic to sum-min. This is because $y_1 \times y_2 = \min(y_1,y_2) \forall y_i \in \{0,1\}$. Therefore, the problem of computing the permanent for such matrices reduces to sum-min
inference in the factor-graph of \cite{huang2009approximating}.
\end{proof}



\subsection{Complexity of general inference classes}
Let $\compclass(.)$ denote the complexity class of an inference class in the hierarchy.
In obtaining the complexity class of problems with $M > 1$, we use the following fact, which is also used in the polynomial hierarchy: $\Poly^{\NP} = \Poly^{\coNP}$~\cite{arora2009computational}. 
In fact $\Poly^{\NP^{\Aclass}} = \Poly^{\coNP^{\Aclass}}$, for any oracle $\Aclass$. 
This means that by adding a polynomial marginalization to the problems in $\Phi_M(\sumset, \pset)$ and $\Psi_M(\sumset, \pset)$,
we get the same complexity class $\compclass(\Delta_{M+1}(\sumset, \pset \cup \{M+1\}))$.
The following gives a recursive definition of complexity class for problems in the inference hierarchy.\footnote{
We do not prove the completeness w.r.t. complexity classes beyond the first level of the hierarchy and only assert the membership.
} Note that the definition of the complexity for each class is very similar to the recursive definition
of members of each class in \refEq{eq:hierarchymembers1} and \refEq{eq:hierarchymembers2}

\begin{theorem} The complexity of inference classes in the hierarchy is given by the recursion
\begin{align}
&\compclass(\Phi_{M+1}(\sumset, \pset)) =  \coNP^{\compclass(\Xi_M(\sumset, \pset) \back \Phi_M(\sumset, \pset))}\label{eq:compmin}\\
&\compclass(\Psi_{M+1}(\sumset, \pset)) =  \NP^{\compclass(\Xi_M(\sumset, \pset) \back \Psi_M(\sumset, \pset))} \label{eq:compmax}\\
&\compclass(\Sigma_{M+1}(\sumset\cup \{M+1\}, \pset)) = \PPclass^{\compclass(\Xi_M(\sumset, \pset) \back \Sigma_M(\sumset, \pset))} \label{eq:compsum}\\
&\compclass(\Delta_{M+1}(\sumset, \pset\cup \{M+1\})) =  \Poly^{\compclass(\Xi_M(\sumset, \pset))} \label{eq:comppoly}
\end{align}
where the base members are defined in \refEq{eq:basemembers} and belong to $\Poly$.
\end{theorem}
\begin{proof}
  Recall that our definition of factor graph ensures that $\qq(\xs)$ can be evaluated in polynomial time and therefore the base members are in $\Poly$ (for complexity of base members of $\Delta$ see \refTheorem{th:sumsum_poly}). We use these classes as the base of our induction and assuming the complexity classes above are correct for $M$, we show that are correct for $M+1$. We consider all the above statements one by one:
  
\noindent \bullitem\ \textit{Complexity for members of} $\Phi_{M+1}(\sumset, \pset)$:\\
 Adding an exponential-sized \emph{min}-marginalization to an inference problem with known complexity $\Aclass$,
 requires a Turing machine to non-deterministically enumerate 
$\zs_{\JJ_M} \in \XX_{\JJ_M}$ possibilities, then call the $\Aclass$
oracle with the ``reduced factor-graph'' -- in which $\xs_{\JJ_M}$ is clamped to $\zs_{\JJ_M}$ -- 
and reject iff any of the calls to oracle rejects. This means 
$\compclass(\Phi_{M+1}(\sumset, \pset)) = \coNP^{\Aclass}$.

Here, \refEq{eq:compmin} is also making another assumption expressed in the following claim.
\begin{claim}\label{claim:inproof}
All inference classes in 
$\Xi_M(\sumset, \pset) \back \Phi_M(\sumset, \pset)$ have the same complexity $\Aclass$.
\end{claim}
\begin{itemize}
  \item $M = 0$: the fact that $\qq(\xs)$ can be evaluated in polynomial time means that $\Aclass = \Poly$.
\item  $M > 0$:  $\Xi_M(\sumset, \pset) \back \Phi_M(\sumset, \pset)$  only contains one inference class -- that is exactly only one of the following cases
is correct:
\begin{itemize}
\item $M \in \sumset \; \Rightarrow \; \Xi_M(\sumset, \pset) \back \Phi_M(\sumset, \pset) = \Sigma_M(\sumset, \pset)$
\item $M \in \pset \; \Rightarrow \; \Xi_M(\sumset, \pset) \back \Phi_M(\sumset, \pset) =  \Delta_M(\sumset, \pset)$
\item $M \notin \sumset \cup \pset \; \Rightarrow \; \Xi_M(\sumset, \pset) \back \Phi_M(\sumset, \pset) =  \Psi_M(\sumset, \pset)$. \\
(in constructing the hierarchy we assume two consecutive marginalizations are distinct and the current marginalization is a minimization.)
\end{itemize}
But if $\Xi_M(\sumset, \pset) \back \Phi_M(\sumset, \pset)$ contains a single class, the inductive hypothesis ensures that all problems in $\Xi_M(\sumset, \pset) \back \Phi_M(\sumset, \pset)$ have the same complexity class $\Aclass$.
\end{itemize}
This completes the proof of our claim.

\noindent \bullitem\ \textit{Complexity for members of} $\Psi_{M+1}(\sumset, \pset)$:\\
Adding an exponential-sized \emph{max}-marginalization to an inference problem with known complexity $\Aclass$,
 requires a Turing machine to non-deterministically enumerate 
$\zs_{\JJ_M} \in \XX_{\JJ_M}$ possibilities, then call the $\Aclass$
oracle with the reduced factor-graph  
and accept iff any of the calls to oracle accepts. 
This means 
$\compclass(\Psi_{M+1}(\sumset, \pset)) = \NP^{\Aclass}$.
Here, an argument similar to that of \refClaim{claim:inproof} ensures that
$\Xi_M(\sumset, \pset) \back \Psi_M(\sumset, \pset)$
in \refEq{eq:compmax} contains a single inference class.

\noindent \bullitem\ \textit{Complexity for members of} $\Sigma_{M+1}(\sumset\cup \{M+1\}, \pset)$:\\
Adding an exponential-sized \emph{sum}-marginalization to an
inference problem with known complexity $\Aclass$,
 requires a Turing machine to non-deterministically enumerate 
$\zs_{\JJ_M} \in \XX_{\JJ_M}$ possibilities, then call the $\Aclass$
oracle with the reduced factor-graph  
and accept iff majority of the calls to oracle accepts. 
This means $\compclass(\Psi_{M+1}(\sumset, \pset)) = \PPclass^{\Aclass}$.
\begin{itemize}
\item $M = 0$: the fact that $\qq(\xs)$ can be evaluated in polynomial time means that $\Aclass = \Poly$.
\item  $M > 0$: 
  \begin{itemize}
\item $M \in \pset \; \Rightarrow \; \Xi_M(\sumset, \pset) \back \Sigma_M(\sumset, \pset) = \Delta_M(\sumset, \pset)$.
  \item $M \notin \pset \cup \sumset \; \Rightarrow \; \Xi_M(\sumset, \pset) \back \Sigma_M(\sumset, \pset) = \Psi_M(\sumset, \pset) \cup \Phi_M(\sumset, \pset)$: 
 despite the fact that $\Aclass = \compclass(\Psi_{M}(\sumset, \pset))$ is different from
$\Aclass' = \compclass(\Phi_{M}(\sumset, \pset))$, since \textit{$\PPclass$ is closed under complement},
which means $\PPclass^\Aclass = \PPclass^{\Aclass}$ and the recursive definition of complexity  \refEq{eq:compsum} remains correct.
  \end{itemize}
\end{itemize}

\noindent \bullitem\ \textit{Complexity for members of} $\Delta_{M+1}(\sumset, \pset\cup \{M+1\})$:\\
Adding a polynomial-sized marginalization to an
inference problem with known complexity $\Aclass$,
 requires a Turing machine to deterministically enumerate 
 $\zs_{\JJ_M} \in \XX_{\JJ_M}$ possibilities in polynomial time,
 and each time call the $\Aclass$ oracle with the reduced factor-graph  
and accept after some polynomial-time calculation. This means 
$\compclass(\Psi_{M+1}(\sumset, \pset)) = \Poly^{\Aclass}$. Here, there are
three possibilities:
\begin{itemize}
\item $M = 0$: here again  $\Aclass = \Poly$.
\item $M \in \sumset \; \Rightarrow \; \Xi_M(\sumset, \pset) = \Sigma_M(\sumset, \pset)$.
  \item $M \in \pset \; \Rightarrow \; \Xi_M(\sumset, \pset) = \Delta_M(\sumset, \pset)$.
   \item $M \notin \pset \cup \sumset \; \Rightarrow \; \Xi_M(\sumset, \pset) = \Psi_M(\sumset, \pset) \cup \Phi_M(\sumset, \pset)$, in which case since $\PPclass^{\NP^{\mathbb{B}}} = \PPclass^{\coNP^{\mathbb{B}}}$, the recursive definition of complexity in \refEq{eq:comppoly} remains correct.
\end{itemize}
\end{proof}

\begin{example}\label{example:marginalmap}
  Consider the marginal-MAP inference of \refEq{eq:marginalmap}.
The decision version of this problem, $\qq(\emptyset) \overset{?}{\geq} q$,
 is a member of $\Psi_{2}(\{1\}, \emptyset)$ which also includes
 $\max-\sumop-\min$ and $\max-\sumop-\max$.
The complexity of this class according to \refEq{eq:compmax} is $\compclass(\Psi^{2}(\{1\}, \emptyset)) = \NP^{\PPclass}$.
However, marginal-MAP is also known to be ``complete'' w.r.t. $\NP^{\PPclass}$~\cite{park2004complexity}.
Now suppose that the max-marginalization over $\xs_{\JJ_2}$ is polynomial (\eg $|\JJ_2|$ is constant). 
Then marginal-MAP belongs to $\Delta_{2}(\{1\}, \{2\})$ with complexity $\poly^{\PPclass}$.
This is because a Turing machine can enumerate all $\zs_{\JJ_2} \in \XX_{\JJ_2}$ in polynomial time and
call its $\PPclass$ oracle to see if 
  \begin{align*}
    &\qq(\xs_{\JJ_0} \mid \zs_{\JJ_2}) \overset{?}{\geq} q\\
     \text{where} \quad &\qq(\xs_{\JJ_0} \mid \zs_{\JJ_2}) = \sum_{\xs_{\JJ_2}} \prod_{\II} \ff_{\II}(\xs_{\II \back \JJ_2}, \zs_{\II \cap \JJ_2}) 
  \end{align*}
and \emph{accept} if any of its calls to oracle accepts, and rejects otherwise.
Here, $\ff_{\II}(\xs_{\II \back \JJ_2}, \zs_{\II \cap \JJ_2})$ is the reduced factor, in which all the variables in $\xs_{\JJ_2}$ are fixed to $\zs_{\JJ_2\cap \II}$.  
\end{example}

Here, \magn{Toda's theorem} \cite{toda1991pp} has an interesting implication w.r.t.~the hierarchy. 
This theorem states that  $\PPclass$ is as hard as the polynomial hierarchy, which means
$\min-\max-\min-\ldots-\max$ inference for an arbitrary, but constant, number of min and max operations appears below the
sum-product inference in the inference hierarchy.

\subsection{Complexity of the hierarchy}\label{sec:logical}
By restricting the domain $\RR$ to $\{0,1\}$, min and max become isomorphic to logical AND ($\wedge$) and OR ($\vee$) respectively,
where $1 \cong \truemath, 0 \cong \falsemath$. By considering the restriction of the inference hierarchy to these
two operations, we can express quantified satisfiability (QSAT) as inference in a graphical model, where $\wedge \cong \forall$ and $\vee \cong \exists$. Let each factor $\ff_\II(\xs_\II)$ be a disjunction --\eg $\ff(\xs_{i,j,k}) = \xx_i \vee \neg \xx_j \vee \neg \xx_k$. Then we have
\begin{align*}
\forall_{\xs_{\JJ_M}} \exists_{\xs_{\JJ_{M-1}}} \ldots \exists_{\xs_{\JJ_2}} \forall_{\xs_{\JJ_1}}  \bigwedge_{\II} \ff_\II(\xs_\II)\; \cong  
\min_{\xs_{\JJ_M}} \max_{\xs_{\JJ_{M-1}}} \ldots \max_{\xs_{\JJ_2}} \min_{\xs_{\JJ_1}} \min_{\II} \ff_\II(\xs_\II)
\end{align*}
By adding the summation operation, we can express the stochastic satisfiability~\cite{littman2001stochastic} and by generalizing the constraints from disjunctions 
we can represent any quantified constraint problem (QCP)~\cite{bordeaux2002beyond}.
QSAT, stochastic SAT and QCPs are all \pspace-complete,
where \pspace\ is the class of problems that can be solved by a (non-deterministic) Turing machine in polynomial space.
Therefore if we can show that inference in the inference hierarchy is in \pspace, 
it follows that inference hierarchy is in \pspace-complete as well.

\begin{algorithm}[h]
\SetKwInOut{Input}{input}\SetKwInOut{Output}{output}
\DontPrintSemicolon
\Input{$\bigopp{M}_{\xs_{\JJ_M}} \bigopp{{M-1}}_{\xs_{\JJ_{M-1}}} \ldots \bigopp{1}_{\xs_{\JJ_1}} \bigotimes_{\II} \ff_\II(\xs_\II)$}
\Output{$\qq(\xs_{\JJ_0})$}
\DontPrintSemicolon
 \For(\tcp{loop over the query domain}){ \textbf{each} $\zs_{\JJ_0} \in \XX_{\JJ_0}$}{
 \For(\tcp{loop over $\XX_{i_N}$}){ \textbf{each} $\zz_{i_N} \in \XX_{i_N}$}{
.\;
.\;
.\;
\For(\tcp{loop over $\XX_{i_1}$}){ \textbf{each} $\zz_{i_{1}} \in \XX_{i_1}$}{
$\qq_{1}(\zz_{i_1}) := \bigotimes_{\II} \ff_\II(\zs_\II)$;
}
$\qq_{i_2}(\zz_{i_2}) := \bigopp{\jfun(i_1)}_{\xx_{i_1}} \qq_1(\xx_{i_1})$\;
.\;
.\;
.\;
$\qq_{N}(\zz_{i_N}) := \bigopp{\jfun(i_{{N-1}})}_{\xx_{i_{{N-1}}}} \qq_{{N-1}}(\xx_{i_{{N-1}}})$\;
}
$\qq(\zs_{\JJ_{0}}) := \bigopp{\jfun(i_{{N}})}_{\xx_{i_N}} \qq_N(\xx_{i_N})$\;
}
\caption{inference in \pspace}\label{alg:pspace_inference}
\end{algorithm}

\begin{theorem}\label{th:hierarchy_complexity}
The inference hierarchy is \pspace-complete.
\end{theorem}

\begin{proof} 
To prove that a problem is \pspace-complete, we have to show that 1) it is in \pspace\ and 2) a \pspace-complete problem reduces to it.
We already saw that QSAT, which is \pspace-complete, reduces to the inference hierarchy.
But it is not difficult to show that inference hierarchy is contained in \pspace.
Let 
  \begin{align*}
    \qq(\xs_{\JJ_0}) = \bigopp{M}_{\xs_{\JJ_M}} \bigopp{M-1}_{\xs_{\JJ_{M-1}}} \ldots \bigopp{1}_{\xs_{\JJ_1}} \bigotimes_{\II} \ff_\II(\xs_\II)
  \end{align*}
be any inference problem in the hierarchy.
We can simply iterate over all values of $\zs \in \XX$ in nested loops or using a recursion.
Let $\jfun(i): \{1,\ldots,N\} \to \{1,\ldots,M\}$ be the index of the marginalization that involves $\xx_i$ -- that is $i \in \JJ_{\jfun(i)}$. Moreover let $i_1,\ldots,i_N$ be an ordering of variable indices
such that $\jfun(i_k) \leq \jfun(i_{k+1})$.
\RefAlgorithm{alg:pspace_inference} uses this notation to demonstrate this procedure using nested loops.
Note that here we loop over individual domains $\XX_{i_k}$ rather than $\XX_{\JJ_m}$ and track
only temporary tuples $\qq_{i_k}$, so that the space complexity remains polynomial in $N$. 

\end{proof}


\section{Polynomial-time inference}\label{sec:gdl}
Our definition of inference 
was based on an expansion operation $\otimes$ and one or more marginalization operations $\opp{1}, \ldots, \opp{M}$.
If we assume only a single marginalization operation, polynomial time inference
is still not generally possible.
However, if we further assume that the expansion operation is distributive
over marginalization and the factor-graph has no loops,
exact polynomial time inference is possible. 
\begin{definition}
A \magn{commutative semiring} $\semiring=(\RR, \oplus, \otimes)$ is the combination of two 
commutative semigroups $\semig_e = (\RR, \otimes)$ and $\semig_m = (\RR, \oplus)$ with two additional properties
\begin{itemize} 
\item identity elements $\identt{\oplus}$ and $\identt{\otimes}$ such that $\identt{\oplus} \oplus a = a$ and $\identt{\otimes} \otimes a = a$. Moreover  $\identt{\oplus}$ is an \magn{annihilator} for $\semig_e = (\otimes, \RR)$: $a \otimes \identt{\oplus} = \identt{\oplus}\quad \forall a \in \RR$.\footnote{
When dealing with reals, this is $\identt{\oplus} = 0$; note $a \times 0 = 0$. Indeed it may be useful to view $\identt{\otimes}$ as $1 \in \Re$ and $\identt{\oplus}$ as $0 \in \Re$.}

\item distributive property: $$a \otimes (b \oplus c) = (a \otimes b) \oplus (a \otimes b)\quad \forall a,b,c \in \RR$$
\end{itemize}
\end{definition}

The mechanism of efficient inference using distributive law can be seen in a simple example: instead of  calculating $ \min(a +  b , a + c)$,
using the fact that summation distributes over minimization, we may instead obtain the same result using $a + \min(b , c)$, which requires fewer operations.

\begin{example}\label{example:semirings}
The following are some examples of commutative semirings:
\begin{easylist}
& Sum-product $(\Re^{\geq 0}, + , \times)$. 
& Max-product $(\Re^{\geq 0} \cup  \{-\infty\}, \max, \times)$ and $(\{0,1\}, \max, \times)$. 
& Min-max $(\settype{S}, \min, \max)$ on any ordered set $\settype{S}$. 
& Min-sum $(\Re \cup \{ \infty \}, \min, +)$ and $(\{0,1\}, \min, +)$. 
& Or-and $(\{\truemath,\falsemath\}, \vee, \wedge)$.
& Union-intersection $(2^{\SS}, \cup, \cap)$ for 
any power-set $2^{\SS}$.
& The semiring of natural numbers with greatest common divisor and least common multiple $(\settype{N},\mathsf{lcm}, \mathsf{gcd})$. 
& Symmetric difference-intersection semiring for any power-set $(2^\SS,\nabla, \cap)$.
\end{easylist}
Many of the semirings above are isomorphic --\eg $y' \cong -\log(y)$ defines an isomorphism between min-sum and max-product. 
It is also easy to show that the or-and semiring is 
isomorphic to min-sum/max-product semiring on $\RR = \{0,1\}$.
\end{example}

The inference problems in the example above have different properties indirectly inherited from their commutative semirings:
for example, the operation $\min$ (also $\max$) is a \magn{choice function},
 which means $\min_{a \in \settype{A}} a \quad \in \settype{A}$. 
The implication is that if the summation ($\oplus$) of the semiring is  
$\min$ (or $\max$), we can replace it with $\arg_{\xs_{\JJ_M}} \max$
and (if required) recover $\qq(\emptyset)$
using $\qq(\emptyset) = \bigotimes_{\II} \ff_\II(\xs^*)$ in polynomial time.


As another example, since both operations have inverses, sum-product is a \magn{field} \cite{pinter2012book}. 
The availability of inverse for $\otimes$ operation -- \ie when $\semig_e$ is an Abelian group -- 
has an important implication for inference:
the expanded form of \refEq{eq:expanded} can be normalized, and we may inquire about \textbf{normalized marginals}
\begin{align}
\quad \pp(\xs_{\JJ}) \quad = \quad &  \bigoplus_{ \xs_{\back \JJ}}  \pp(\xs)&\label{eq:semiring_marginalization}\\
\text{where}\quad \pp(\xs) \quad \defeq \quad & \frac{1}{\qq(\emptyset)} \otimes \big (\bigotimes_{\II}\ff_{\II}(\xs_{\II}) \big ) &\quad \text{if} \quad \qq(\emptyset) \neq \identt{\oplus}\label{eq:p}\\
\pp(\xs) \quad \defeq \quad \identt{\oplus} &\quad \text{if} \quad \qq(\emptyset) = \identt{\oplus} \label{eq:specialp}
\end{align}
where $\pp(\xs)$ is the normalized joint form. We deal with the case where the integral evaluates to the annihilator as a special case because
division by annihilator may not be well-defined.
This also means, when working with normalized expanded form and normalized marginals, we always have
$\bigoplus_{\xs_{\JJ}} \pp(\xs_{\JJ}) = \identt{\otimes}$ 
\begin{example}
Since $\semig_{e} = (\Re^{>0}, \times)$ and $\semig_{e} = (\Re, +)$
are both Abelian groups, min-sum and sum-product inference have normalized marginals. 
For min-sum inference this means $\min_{\xs_{\JJ}}\pp(\xs_{\JJ}) = \identt{\sumop} = 0$. However, for min-max inference, since $(\SS, \max)$ is not Abelian, normalized marginals are not defined.
\end{example}

We can apply the identity and annihilator of a commutative semiring to define constraints.
\begin{definition}\label{def:constraint}
A \magn{constraint} is a factor $\ff_\II: \XX_{\II} \to \{\identt{\otimes}, \identt{\oplus}\}$
whose range is limited to identity and annihilator of the expansion monoid.\footnote{Recall that a monoid is a semigroup with an identity. 
The existence of identity here is a property of the semiring.}
\end{definition}
Here, $\ff_{\II}(\xs) = \identt{\oplus}$ iff $\xs$ is forbidden and 
$\ff_{\II}(\xs) = \identt{\otimes}$ iff it is permissible. 
A \magn{constraint satisfaction problem} (CSP) 
is any inference problem on a semiring
in which all factors are constraints. Note that this allows definition 
of the ``same'' CSP on any commutative semiring and also indicates that inference in general semirings should be difficult. The following theorem formalizes this intuition.
\begin{theorem}\label{th:semiring_inference}
Inference in any commutative semiring is \NP-hard under randomized polynomial-time reduction.
\end{theorem}
\begin{proof}
To prove that inference in any semiring $\semiring = (\RR, \identt{\oplus}, \identt{\otimes})$ is \NP-hard under randomized polynomial
 reduction, we deterministically reduce \textit{unique satisfiability} (USAT) to an inference problems on any semiring.
USAT is a so-called ``promise problem'', that asks whether a satisfiability problem that is promised to have either zero or one satisfying assignment is satisfiable.
\cite{valiant1986np} prove that a polynomial time randomized algorithm (\RP) for USAT
implies that \RP=\NP.

For this reduction consider a set of binary variables $\xs \in \{0,1\}^{N}$, 
one per each variable in the given instance of USAT.
For each clause, define a constraint factor $\ff_{\II}$ such that $\ff_{\II}(\xs_{\II}) = \identt{\otimes}$ if $\xs_{\II}$ satisfies that clause and  $\ff_{\II}(\xs_{\II}) = \identt{\oplus}$ otherwise. This means, $\xs$ is a satisfying assignment for USAT iff $\qq(\xs) = \bigotimes_{\II} \ff_\II(\xs_{\II}) = \identt{\otimes}$. If the instance is unsatisfiable, the integral $\qq(\emptyset) = \bigoplus_{\xs} \identt{\oplus} = \identt{\oplus}$ (by definition of $\identt{\oplus}$).
If the instance is satisfiable there is only a single instance $\xs^*$ for which $\qq(\xs^*) = \identt{\otimes}$, and therefore the integral evaluates to $\identt{\otimes}$.
Therefore we can decide the satisfiability of USAT by performing inference on any semiring, by only relying on the properties of identities.  

\end{proof}
\begin{example}
Inference on xor-and semiring $(\{\truemath,\falsemath\}, \xor, \wedge)$, where each factor has
a disjunction form, is called parity-SAT, 
which asks whether the number of SAT solutions is even or odd. A corollary to theorem~\ref{th:semiring_inference} is that parity-SAT is \NP-hard under randomized reduction,
which is indeed the case \cite{valiant1986np}.
\end{example}
We find it useful to use the same notation for the 
 \magn{identity function} $\ident(\mathrm{condition})$ -- \eg
\begin{align}\label{eq:identity}
  \ident(\mathrm{cond.}) \; \defeq \;  \left \{ 
\begin{array}{r c c c}
 &(+,\times) & (\min,+) & (\min,\max)\\
\mathrm{cond.} = \truemath & 1 & 0 & -\infty\\
\mathrm{cond.} = \falsemath & 0 & +\infty & +\infty\\
\end{array}   
\right .
\end{align}
where the intended semiring for $\ident(.)$ function will be clear from the context.

\subsection{Distributive law and its limits}\label{sec:bp}
A naive approach to inference over commutative semirings  
\begin{align}\label{eq:semiringinference}
\qq(\xs_{\JJ}) \quad = \quad \bigoplus_{\xs_{\back \JJ}} \bigotimes_{\II} \ff_{\II}(\xs_\II) 
\end{align}
or its normalized version (\refEq{eq:semiring_marginalization}), involves constructing a complete $N$-dimensional array of $\qq(\xs)$
using the tensor product $\qq(\xs) \; = \; \bigotimes_\II \ff_\II(\xs_\II)$ and then perform $\oplus$-marginalization.
However, the number of elements in $\qq(\xs)$ is $\vert \XX \vert$, which is exponential in $N$, the number of variables.

If the factor-graph is loop free, 
we can use distributive law to make inference tractable. 
Assuming $\qq(\xs_\KK)$ (or $\qq(\xx_k)$) is the marginal of interest, form a tree with $\KK$ (or $k$) as its root.
Then starting from the leaves, using the distributive law, we can move the $\oplus$ inside the $\otimes$ 
 and define ``messages'' from leaves towards the root as follows:
\begin{align}
\msgq{i}{{{\II}}}(\xx_i) \quad & = \quad  \bigotimes_{\JJ \in \nb{i} \back \II} \msgq{{\JJ}}{i}(\xx_i) \label{eq:miI_semiring}\\
\msgq{{{\II}}}{i}(\xx_i) \quad & = \quad  \bigoplus_{ \xs_{\back i}} \ff_{{\II}}(\xs_{\II})
\bigotimes_{j \in \nb \II \back i} \msgq{j}{{\II}}(\xx_{j})  \label{eq:mIi_semiring}
\end{align}
where \refEq{eq:miI_semiring} defines the message from a variable to a factor, 
closer to the root and similarly \refEq{eq:mIi_semiring} defines the message from factor ${\II}$ to a variable $i$ closer to the root.
Here, the distributive law allows moving the $\bpplus$ over the domain $\XX_{\II \back i}$ from outside to inside of \refEq{eq:mIi_semiring} --
the same way $\oplus$ moves its place in $(a \otimes b) \oplus (a \otimes c)$ to give $a \otimes (b \oplus c)$, where 
$a$ is analogous to a message.

\begin{figure}
\centering
\includegraphics[width=.5\textwidth]{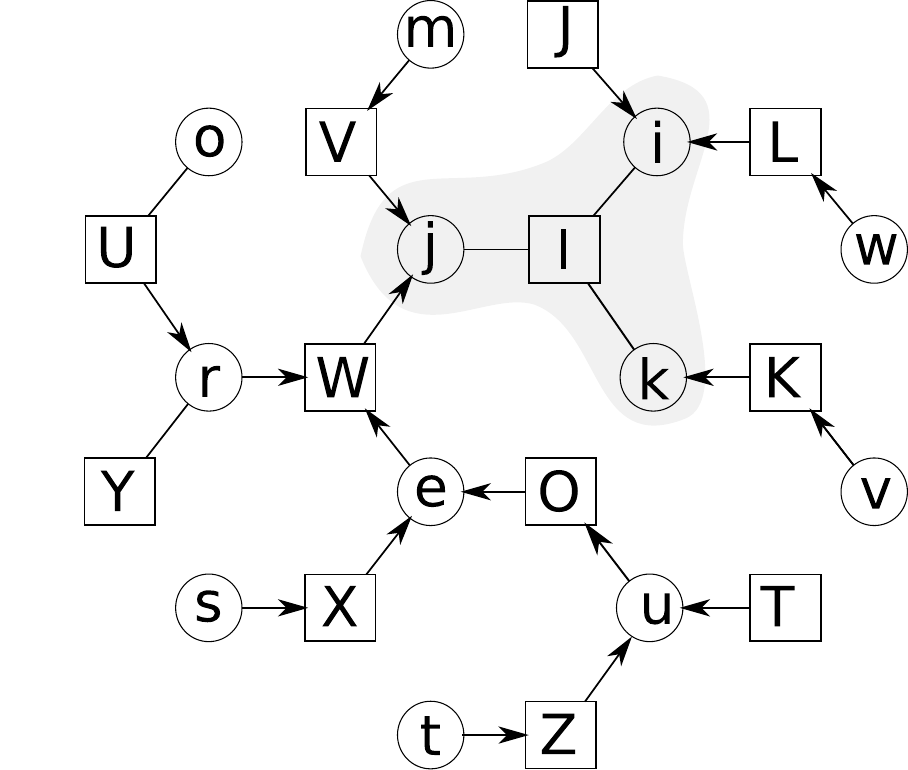}
\caption[Belief Propagation on a loop-free factor-graph]{The figure shows a loop-free factor-graph and the direction of 
messages sent between variable and factor nodes in order to calculate the marginal over the grey region.}
\label{fig:bp-tree}
\end{figure}

By starting from the leaves, and calculating the messages towards the root, we obtain 
the marginal over the root node as the product of incoming messages
\begin{align}
\qq(\xx_{k}) \quad &= \quad  \bigotimes_{\II \in \nb k} \msgq{\II}{k}(\xx_k)
\end{align}
In fact, we can assume any
subset of variables $\xs_{\AAA}$ (and factors within those variables) to be the root. Then, the set of all incoming messages to $\AAA$, 
produces the marginal 
\begin{align}\label{eq:bpmarg_region}
\qq(\xs_{\AAA}) \quad = \quad \left ( \bigotimes_{\II \subseteq \AAA} \ff_{\II}(\xs_{\II}) \right ) \left (\bigotimes_{i \in \AAA, \JJ \in \nb i, \JJ \not \subseteq \AAA} \msgq{\JJ}{i}(\xs_i) \right )
\end{align}

\begin{example}
Consider the joint form represented by the factor-graph of  \refFigure{fig:bp-tree} 
\begin{align*}
  \qq(\xs) \quad = \quad \bigotimes_{\AAA \in \{\II,\JJ,\KK, \mathrm{L},\mathrm{O},\mathrm{T},\mathrm{U},\mathrm{V},\mathrm{W},\mathrm{X},\mathrm{Y},\mathrm{Z}\}}  \ff_{\AAA}(\xs_{\AAA}) 
\end{align*}
and the problem of calculating the marginal over $\xs_{\{i,j,k\}}$ (\ie the shaded region).
\begin{align*}
  \qq(\xs_{\{i,j,k\}}) \quad = \quad \bigoplus_{\xs_{\back \{i,j,k\}}} \quad \bigotimes_{\AAA \in \{\II,\JJ,\KK, \mathrm{L},\mathrm{O},\mathrm{T},\mathrm{U},\mathrm{V},\mathrm{W},\mathrm{X},\mathrm{Y},\mathrm{Z}\}}  \ff_{\AAA}(\xs_{\AAA}) 
\end{align*}

We can move the $\oplus$ inside the $\otimes$ to obtain
\begin{align*}
 \qq(\xs_{\{i,j,k\}}) \quad = \quad  \ff_{\II}(\xs_{\II}) \bptimes \msgq{\mathrm{L}}{i}(\xx_i) \bptimes \msgq{\mathrm{K}}{i}(\xx_i) \bptimes \msgq{\mathrm{V}}{j}(\xx_j) \bptimes \msgq{\mathrm{W}}{j}(\xx_j) \bptimes \msgq{\mathrm{K}}{k}(\xx_k)
\end{align*}
where each term $\msgq{\AAA}{i}$ factors the summation on the corresponding sub-tree.
For example 
\begin{align*}
  \msgq{\mathrm{L}}{i} \quad = \quad \bigoplus_{\xx_w} \ff_{\mathrm{L}}(\xs_{\mathrm{L}})
\end{align*}

Here the message $\msgq{\mathrm{W}}{j}$ is itself a computational challenge
\begin{align*}
  \msgq{\mathrm{W}}{j} \quad = \quad \bigoplus_{\xs_{\back j}} \bigotimes_{\AAA \in \{\mathrm{W}, \mathrm{U}, \mathrm{Y}, \mathrm{X}, \mathrm{O}, \mathrm{T}, \mathrm{Z} \} }\ff_{\mathrm{\AAA}}(\xs_{\AAA})
\end{align*}

However we can also decompose this message over  sub-trees
\begin{align*}
  \msgq{\mathrm{W}}{j} \quad = \quad \bigoplus_{\xs_{\back j}} \ff_{\mathrm{\AAA}}(\xs_{\AAA})
\bptimes \msgq{e}{\mathrm{W}}(\xx_{e}) \bptimes \msgq{r}{\mathrm{W}}(\xx_{r})
\end{align*}
where again, using the distributive law, $\msgq{e}{\mathrm{W}}$ and  $\msgq{r}{\mathrm{W}}$ 
further simplify based on the incoming messages to the variable nodes $\xx_r$ and $\xx_e$.
\end{example}

This procedure is known as Belief Propagation (BP), which is sometimes prefixed with the corresponding semiring -- \eg sum-product BP.
Even though BP is only guaranteed to produce correct answers when the factor-graph is a
tree (and few other cases \cite{Aji1998,Weiss2001a,bayati2005maximum,weller2013map}), 
it performs surprisingly well when applied as a fixed point iteration to graphs with loops \cite{Murphy1999,gallager1962low}. In the case of loopy graphs, the message updates are repeatedly applied in the hope of convergence. 
Here, for numerical stability, when the $\otimes$ operator has an inverse, the messages are normalized. We use $\propto$ to indicate this normalization according to the mode of inference  
\begin{align}
\msg{{{\II}}}{i}(\xx_i) \quad & \propto \quad   \bigoplus_{ \xs_{\back i}} \ff_{{\II}}(\xs_{\II})
\bigotimes_{j \in \nb \II \back i} \msg{j}{{\II}}(\xx_{j})  & \;\propto\; \MSG{\II}{i}(\msgss{\nb \II \back i}{\II})(\xx_i)\label{eq:mIi_semiring_norm}\\
\msg{i}{{{\II}}}(\xx_i) \quad & \propto \quad   \bigotimes_{\JJ \in \nb{i} \back \II} \msg{{\JJ}}{i}(\xx_i)  &\;\propto\; \MSG{i}{\II}(\msgss{\nb i \back \II}{i})(\xx_i)\label{eq:miI_semiring_norm}\\
\ph(\xs_{\II}) \quad &\propto \quad   \ff_{\II}(\xs_{\II}) \bigotimes_{i \in \nb \II} \msgs{i}{\II}(\xx_i)  &\quad \label{eq:marg_semiring_factor_norm}\\
\ph(\xx_{i}) \quad & \propto \quad    \bigotimes_{\II \in \nb i} \msg{\II}{i}(\xx_i)  &\quad \label{eq:marg_semiring_norm}
\end{align}
Here, for general graphs, $\ph(\xx_i)$ and $\ph(\xs_{\II})$ are approximations to $\pp(\xx_i)$
and $\pp(\xs_\II)$ of \refEq{eq:semiring_marginalization}.
The \textbf{functionals} $\MSG{i}{\II}(\msg{\nb i \back \II}{i})(.)$ and 
$\MSG{\II}{i}(\msgss{\nb \II \back i}{\II})(.)$ cast the BP message updates as an operator on a subset of incoming messages -- \ie $\msgss{\nb i \back \II}{i} = \{ \msg{\JJ}{i} \mid \JJ \in \nb i \back \II \}$. We use these functional notation in presenting the algebraic form of survey propagation in \refSection{sec:sp}.

\subsection{The limits of message passing}\label{sec:limits}
By observing the application of distributive law in semirings, a natural question to ask is: can we use distributive law for polynomial time inference on loop-free graphical models over any of the inference problems at higher levels of inference hierarchy or in general, for any inference problem with more than one marginalization operation?
The answer to this question is further motivated by the fact that, when loops exists, the same scheme may become a powerful approximation technique. 
When we have more than one marginalization operations, a natural assumption in using distributive law 
is that the expansion operation distributes over all the marginalization operations -- \eg  as in min-max-sum (where sum distributes over both min and max), min-max-min, xor-or-and.
\index{marginalization!two operations}
\index{inference!efficient}
Consider the simplest case with three operators $\opp{1}$, $\opp{2}$ and $\otimes$, where $\otimes$ distributes over both $\opp{1}$ and $\opp{2}$.
Here the integration problem is
$$\qq(\emptyset) \quad = \quad \bigopp{2}_{\xs_{\JJ_2}} \bigopp{1}_{\xs_{\JJ_1}} \bigotimes_{\II} \ff_{\II}(\xs_\II)$$
where $\JJ_1$ and $\JJ_2$ partition $\{1,\ldots,N\}$.

In order to apply distributive law for each pair $(\opp{1}, \otimes)$ and $(\opp{2}, \otimes)$, we need to be able to commute $\opp{1}$ and $\opp{2}$ operations.
That is, we require 
\begin{align}\label{eq:op_commute}
\bigopp{1}_{\xs_{\AAA}} \bigopp{2}_{\xs_{\BBB}} \fg(\xs_{\AAA \cup \BBB}) = \bigopp{2}_{\xs_\BBB} \bigopp{1}_{\xs_{\AAA}} \fg(\xs_{\AAA \cup \BBB}).
\end{align}
for the specified $\AAA \subseteq \JJ_1$ and $\BBB \subseteq \JJ_2$.

Now, consider a simple case involving two binary variables $\xx_i$ and $\xx_j$, where $\fg(\xs_{\{i,j\}})$ is 
\vspace{.05in}
\begin{center}
\scalebox{.8}{
\begin{tabu}{r r |[2pt] c | c }
\multicolumn{2}{c}{}&\multicolumn{2}{c }{$\xx_j$}\\
\multicolumn{2}{c}{}&0&1\\\tabucline[2pt]{3-4}
\multirow{2}{*}{$\xx_i$}&0&a&\multicolumn{1}{c |[2pt]}{b}\\\cline{2-4}
&1&c&\multicolumn{1}{c |[2pt]}{d}\\\tabucline[2pt]{3-4}
\end{tabu}
}
\end{center}
\vspace{.05in}
Applying \refEq{eq:op_commute} to this simple case (\ie $\AAA = \{i\}, \BBB=\{j\}$), we require
$$
  (a \opp{1} b) \opp{2} (c \opp{1} d) \quad = \quad (a \opp{2} b) \opp{1} (c \opp{2} d). 
$$

The following theorem leads immediately to a  negative result:
\begin{theorem}\cite{eckmann1962group}:
\begin{align*}
  (a \opp{1} b) \opp{2} (c \opp{1} d) \; = \; (a \opp{2} b) \opp{1} (c \opp{2} d)
\quad \Leftrightarrow \quad \opp{1} = \opp{2} \quad \forall a,b,c
\end{align*}
\end{theorem}
\noindent which implies that \emph{direct application of distributive law to tractably and exactly solve any  inference problem with more than one marginalization operation is unfeasible, even for tree structures.} This limitation was previously known for marginal MAP inference~\cite{park2004complexity}.
\index{marginal MAP}

Min and max operations have an interesting property in this regard.
Similar to any other operations for min and max we have
\begin{align*}
  \min_{\xs_{\JJ}} \max_{\xs_{\II}} \fg(\xs_{\II \cup \JJ}) \neq \max_{\xs_{\II}} \min_{\xs_{\JJ}} \fg(\xs_{\II \cup \JJ})
\end{align*}

However,
\index{minimax theorem}
\index{game theory}
\index{graphical games}
\index{graphical models!game theory}
 if we slightly change the inference problem (from pure assignments $\xs_{\JJ_l}  \in \XX_{\JJ_l}$ to a distribution over assignments; \aka mixed strategies), as a result of the celebrated \emph{minimax theorem} \cite{von2007theory}, the min and max operations commute -- \ie
\begin{align*}
   \min_{\strat(\xs_{\JJ})} \max_{\strat(\xs_{\II})} \sum_{\xs_{\II \cup \JJ}} \strat(\xs_{\JJ}) \fg(\xs_{\II \cup \JJ}) \strat(\xs_{\II}) \quad = \quad \max_{\strat(\xs_{\II})} \min_{\strat(\xs_{\JJ})} \sum_{\xs_{\II \cup \JJ}}  \strat(\xs_{\II}) \fg(\xs_{\II \cup \JJ}) \strat(\xs_{\JJ_1})
\end{align*}
where $\strat(\xs_{\JJ_1})$ and $\strat(\xs_{\JJ_2})$ are mixed strategies.
This property has enabled addressing problems with min and max marginalization operations using message-passing-like procedures. For example, \cite{ibrahimi2011robust} solve this (mixed-strategy) variation of min-max-product inference. Message passing procedures that operate on graphical models for game theory  (\aka ``graphical games'')  also rely on this property~\cite{ortiz2002nash,kearns2007graphical}.

\section{Algebra of survey propagation}\label{sec:sp}
Survey propagation (SP) was first introduced as a message passing solution to satisfiability~\cite{braunstein_survey_2002} and was later generalized to general CSP
\cite{braunstein_constraint_2002} and arbitrary inference problems over factor-graphs~\cite{Mezard09}.
Several works offer different interpretations and generalizations of survey propagation \cite{Kroc2002,braunstein_survey_2003,maneva_new_2004}.
Here, we propose a generalization based the same notions that extends the application of BP to arbitrary commutative semirings.
Our derivation generalizes the variational approach of  
\cite{Mezard09}, in the same way that the algebraic approach to BP (using commutative semirings) generalizes the variational derivation of sum-product and min-sum BP.

As a fixed point iteration procedure,
if BP has more than one fixed points,
it may not converge at all. Alternatively, if the messages are initialized properly, BP may converge to one of its fixed points.
SP equations take ``all'' BP fixed points into account.
We view this task, of dealing with all fixed points, as using a third binary operation 
$\spplus$ on $\RR$. In particular, we require that $\otimes$ also distribute over $\spplus$, forming 
a second commutative semiring. We refer to this new semiring as a \magn{SP semiring}. 
To better explain the role of the third operation we need to introduce some notation.

Let $\msgss{\cdot}{\cdot}$ be a BP fixed point -- that is let
$$
\msgss{\cdot}{\cdot} = \{\, \msg{i}{\II} = \MSG{i}{\II}(\msgss{\nb i\back \II}{i}),
 \msg{\II}{i} = \MSG{\II}{i}(\msgss{\nb \II \back i}{\II})\, \mid \mid \forall i, \II \in \nb i \,\}
$$
and denote the set of all such fixed points by $\WW$.
Each BP fixed point corresponds to an approximation to the integral $\qq(\emptyset)$,
which we denote by $\QQ(\msgss{\cdot}{\cdot})(\emptyset)$ -- using this functional form
to emphasize the dependence of this approximation on BP messages.
Recall that in the original problem, $\XX$ is the domain of assignments, 
$\qq(\xs)$ is the  expanded form and $\bpplus$-marginalization is (approximately) performed by BP.
In the case of survey propagation, $\WW$ is the domain of assignments and the integral $\QQ(\msgss{\cdot}{\cdot})(\emptyset)$ 
evaluates a particular assignment $\msgss{\cdot}{\cdot}$ to all the messages -- \ie 
$\QQ(\msgss{\cdot}{\cdot})(\emptyset)$  is the new expanded form.

In this algebraic perspective, SP efficiently performs a second integral  using $\spplus$ over all fixed points:
\begin{align}\label{eq:sp_q_intmarg}
\QQ(\emptyset)(\emptyset) = \bigspplus_{\msgss{\cdot}{\cdot} \in \WW} \QQ(\msgssq{\cdot}{\cdot})(\emptyset)
\end{align}
\RefTable{table:sp} summarizes this correspondence. It is not immediately obvious how this second integration can be usefull in practice.
The following examples attempts to motivate SP before we explain its derivation.

\begin{example}
If the $\otimes$ operator of the semiring has an inverse, any BP fixed point $\msgss{\cdot}{\cdot}$ represents a joint form $\ph(\xs)$. We can explicitly write this distribution
based on BP marginals (ergo BP messages):
\begin{align}\label{eq:semiring_reparam}
  \ph(\xs) \quad = \quad \frac{\bigbptimes_\II \ph(\xs_\II)}{\bigbptimes_i \big ( \ph(\xx_i) \powerop ({ \vert \nb i \vert - 1 }) \big )} \quad = \quad \PH(\msgss{\cdot}{\cdot})(\xs) 
\end{align}
where the inverse is w.r.t $\otimes$ and the \textit{exponentiation operator} is defined as $a \powerop b \defeq \underbrace{a \otimes \ldots \otimes a}_{b\; \text{times}}$. Our notation using $\PH(\msgss{\cdot}{\cdot})(\xs)$ is to explicitly show the dependence of this joint form on BP messages.
To see why this is correct, we use the exactness of BP on trees and 
 substitute BP marginals \refEqs{eq:marg_semiring_norm}{eq:marg_semiring_factor_norm}
into \refEq{eq:semiring_reparam}:
\begin{align*}
  &\frac{\bigbptimes_\II \ph(\xs_\II)}{\bigbptimes_i \big ( \ph(\xx_i) \powerop ({ \vert \nb i \vert - 1 }) \big )} & \quad = \quad 
  &\frac{\bigbptimes_\II  \ff_{\II}(\xs_{\II}) \bigotimes_{i \in \nb \II} \msg{i}{\II}(\xx_i)}{\bigbptimes_i 
\big ( \bigotimes_{\II \in \nb i} \msg{\II}{i}(\xx_i) \powerop ({ \vert \nb i \vert - 1 }) \big )} & = \\
&\frac{\bigbptimes_\II  \ff_{\II}(\xs_{\II}) \bigotimes_{i \in \nb \II} \msg{i}{\II}(\xx_i)}{\bigbptimes_i 
\big ( \bigotimes_{\II \in \nb i} \msg{i}{\II}(\xx_i) \big )} & \quad = \quad
&\bigbptimes_\II  \ff_{\II}(\xs_{\II}) \quad = \quad  \pp(\xs)
\end{align*}

In some interesting settings for sum-product inference, BP has many ``disjoint'' fixed points, in the sense that $\ph(\xs)$ estimates (as defined above) have small overlap with each other (\eg citep{Mezard1987}) -- \ie $\pp(\xs) \approx \sum_{\msgss{\cdot}{\cdot} \in \WW} \PH(\msgss{\cdot}{\cdot})(\xs)$. In this setting the BP estimate of the integral (and marginals) is inaccurate. However, summing over all such BP integrals gives a more accurate estimate of $\qq(\emptyset)$. This is an application of SP where $\bpplus = \spplus = +$ in \refEq{eq:sp_q_intmarg}. 
\end{example}

\begin{table}
\tbl{The correspondance between BP and SP}{
\begin{tabu}{ r|[2pt]l}
\textbf{Belief Propagation} & \textbf{Survey Propagation} \\
\multicolumn{2}{l}{domain:}\\
$\xs$ & $\msgss{\cdot}{\cdot}$ \\ 
$\forall i \quad \xx_i$ & $\msg{i}{\II}\;,\; \msg{\II}{i} \quad \forall i, \II \in \nb i$  \\ 
$\XX$ & $\WW$ \\
\multicolumn{2}{l}{expanded form:}\\
$\qq(\xs)$ & $\QQ(\msgss{\cdot}{\cdot})(\emptyset)$\\
\multicolumn{2}{l}{integration:}\\
 $\qq(\emptyset) = \bigbpplus_{\xs} \qq(\xs)$ & $\QQ(\emptyset)(\emptyset) = \bigspplus_{\msgss{\cdot}{\cdot}} \QQ(\msgss{\cdot}{\cdot})(\emptyset)$\\
\multicolumn{2}{l}{marginalization:}\\
 $\pp(\xx_i) \propto \bigbpplus_{\xs \back i} \pp(\xs)$ & $\PSP(\msg{\II}{i}) \propto \bigspplus_{\back \msgss{\II}{i}} \PP(\msgss{\cdot}{\cdot})$\\
\multicolumn{2}{l}{factors:}\\
$ \forall \II \quad \ff_\II(\xs_\II) $ & $\PHT_\II(\msgss{\nb \II}{\II})(\emptyset)$, $\PHT_i(\msgss{\nb i \back \II}{i})(\emptyset)$ and $\MSGBIT{i}{\II}(\msg{i}{\II},\msg{\II}{i})(\emptyset)^{-1} \quad \forall i, \II \in \nb i$
\end{tabu}
}
\label{table:sp}
\end{table}


Our derivation of SP requires $(\RR, \bptimes)$ to be an Abelian group (\ie every element of $\RR$ has an inverse w.r.t.~$\bptimes$).
We require $\bptimes$ to be invertable as we need to work with normalized BP
and SP messages.
In \refSection{sec:uniformsp} we introduce another variation of SP that simply counts the 
BP fixed points and relaxes this requirement.


\subsection{Decomposition of the integral}
\label{sec:decompose_integral}
In writing the normalized BP equations in \refSection{sec:bp}, we hid the normalization constant using $\propto$ sign.
Here we explicitly define the normalization constants or \magn{local integrals} by defining unnormalized messages, based on their normalized version
\begin{align}
\msgt{\II}{i}(\xx_i) \quad &\defeq \quad   \bigoplus_{ \xs_{\back i}} \ff_{{\II}}(\xs_{\II}) 
\bigotimes_{j \in \nb \II \back i} \msg{j}{{\II}}(\xx_{j}) \quad &\defeq \quad \MSGT{\II}{i}(\msgss{\nb \II \back i}{\II})(\xx_i) \label{eq:mIi_partition}\\
\msgt{i}{\II}(\xx_i)  \quad &\defeq \quad   \bigotimes_{\JJ \in \nb{i} \back \II} \msg{{\JJ}}{i}(\xx_i)  \quad &\defeq \quad  \MSGT{i}{\II}(\msgss{\nb i \back \II}{i})(\xx_i)\label{eq:miI_partition}\\
\pht_\II(\xs_\II)  \quad &\defeq \quad   \ff_{\II}(\xs_{\II}) \bigotimes_{i \in \nb \II} \msgss{i}{\II}(\xx_i)   \quad &\defeq \quad \PHT_\II(\msgss{\nb \II}{\II})(\xs_\II) \label{eq:marg_factor_partition}\\
\pht_i(\xx_i) \quad  &\defeq \quad  \bigotimes_{\II \in \nb i} \msg{\II}{i}(\xx_i)  \quad &\defeq \quad  \PHT_i(\msgss{\nb i}{i})(\xx_i)  \label{eq:marg_partition}
\end{align}
where each update also has a functional form on the r.h.s.
In each case, the local integrals are simply the integral of unnormalized messages or marginals -- \eg $\msgt{\II}{i}(\emptyset) = \bigoplus_{\xx_i} \msgt{\II}{i}(\xx_i)$.

Define the functional $\MSGBIT{i}{\II}(\msg{i}{\II},\msg{\II}{i})$ as the product of messages from $i$ to $\II$ and vice versa
\begin{align} \label{eq:bi_direction}
  \msgbit{i}{\II}(\xx_i) \quad \defeq \quad  \msg{i}{\II}(\xx_i) \otimes \msg{\II}{i}(\xx_i) \quad \defeq \quad \MSGBIT{i}{\II}(\msg{i}{\II},\msg{\II}{i})(\xx_i)
\end{align}


\begin{theorem}\label{th:integral_decompose} 
If the factor-graph has no loops and $(\RR, \otimes)$ is an Abelian group, the global integral decomposes to local BP integrals as
  \begin{align}\label{eq:partition_decomp_nofunc}
    \qq(\emptyset) \quad = \quad \bigotimes_{\II} \pht_\II(\emptyset) \;\bigotimes_i\pht_i(\emptyset)
\left (\bigotimes_{i, \II \in \nb i} \msgbit{i}{\II}(\emptyset) \right )^{-1}  
  \end{align}
or in other words $\qq(\emptyset)  =  \QQ(\msgss{\cdot}{\cdot})(\emptyset)$ where
  \begin{align}\label{eq:partition_decomp}
    \QQ(\msgss{\cdot}{\cdot})(\emptyset) = \bigotimes_{\II} \PHT_\II(\msgss{\nb \II}{\II})(\emptyset) \;\bigotimes_i\PHT_i(\msgss{\nb i}{i})(\emptyset)
 \left ( \bigotimes_{i, \II \in \nb i} \MSGBIT{i}{\II}(\msg{i}{\II},\msg{\II}{i})(\emptyset) \right)^{-1} 
  \end{align}
\end{theorem}
\begin{proof} 
For this proof we build a tree around an  root node $r$
that is connected to one factor. (Since the factor-graph is a tree such a node always exists.) 
Send BP messages from the leaves, up towards the root $r$ and back to the leaves. 
Here, any message $\msgq{i}{\II}(\xx_i)$,
can give us the integral for the sub-tree that contains all the nodes and factors up to node $i$ using $\msgq{i}{\II}(\emptyset) = \bigbpplus_{\xx_i}\msgq{i}{\II}(\xx_i)$.
Noting that the root is connected to exactly one factor, the global integral is 
\begin{align}\label{eq:genesis}
\bigbpplus_{\xx_r} \qq(\xx_r) = \bigbpplus_{\xx_r} \bigbptimes_{\II \in \nb r} \msgq{\II}{r}(\xx_r) = 
\msgq{\II}{r}(\emptyset)
\end{align}

On the other hand, we have the following relation between $\msgq{i}{\II}$ and $\msg{i}{\II}$ (also corresponding factor-to-variable message)
\begin{align}
  \msgq{i}{\II}(\xx_i) &= \msg{i}{\II}(\xx_i) \bptimes \msgq{i}{\II}(\emptyset) \quad \forall i, \II \in \nb i \label{eq:integral_subs1}\\
  \msgq{\II}{i}(\xx_i) &= \msg{\II}{i}(\xx_i) \bptimes \msgq{\II}{i}(\emptyset) \quad \forall i, \II \in \nb i \label{eq:integral_subs2}
\end{align}

Substituting this into BP \refEqs{eq:miI_semiring}{eq:mIi_semiring} we get
\begin{align}
\msgq{i}{{{\II}}}(\xx_i) \quad & = \quad  \bigotimes_{\JJ \in \nb{i} \back \II} \msgq{\JJ}{i}(\emptyset) \msg{{\JJ}}{i}(\xx_i) \label{eq:step1_1}\\
\msgq{{{\II}}}{i}(\xx_i) \quad & = \quad  \bigoplus_{ \xs_{\back i}} \ff_{{\II}}(\xs_{\II}) 
\bigotimes_{j \in \nb \II \back i} \msgq{j}{{\II}}(\emptyset) \msg{j}{{\II}}(\xx_{j})  
\label{eq:step1_2}
\end{align}
By summing over both l.h.s and r.h.s in equations above and substituting from \refEq{eq:miI_partition}, we get
\begin{align}
\bigbpplus_{\xx_i} \msgq{i}{{{\II}}}(\xx_i) \quad & = \quad  \left ( \bigotimes_{\JJ \in \nb{i} \back \II} \msgq{\JJ}{i}(\emptyset) \right ) \otimes \left (\bigbpplus_{\xx_i} \bigotimes_{\JJ \in \nb{i} \back \II}  \msg{{\JJ}}{i}(\xx_i) \right ) \Rightarrow \notag \\
\msgq{i}{{{\II}}}(\emptyset) \quad & = \quad  \msgt{i}{\II}(\emptyset)\; \bigotimes_{\JJ \in \nb{i} \back \II} \msgq{\JJ}{i}(\emptyset) \label{eq:recurse_1}
\end{align}
and similarly for \refEq{eq:step1_2} using integration and substitution from \refEq{eq:mIi_partition} we have
\begin{align}
\bigbpplus_{\xx_i} \msgq{{{\II}}}{i}(\xx_i) \quad & = \quad  \left ( \bigotimes_{j \in \nb \II \back i} \msgq{j}{{\II}}(\emptyset) \right) \otimes \left( \bigoplus_{ \xs_{\II}} \ff_{{\II}}(\xs_{\II}) 
\bigotimes_{j \in \nb \II \back i} \msg{j}{{\II}}(\xx_{j})  \right ) \Rightarrow \notag \\
\msgq{{{\II}}}{i}(\emptyset) \quad & = \quad   \msgt{\II}{i}(\emptyset)
\bigotimes_{j \in \nb \II \back i} \msgq{j}{{\II}}(\emptyset)   \label{eq:recurse_2}
\end{align}

\RefEqs{eq:recurse_1}{eq:recurse_2} are simply recursive integration on a tree, where the integral
up to node $i$ (\ie $\msgq{i}{{{\II}}}(\emptyset)$ in \refEq{eq:recurse_1}) is reduced to the integral over its sub-trees.
By unrolling this recursion we see that $\msgq{i}{{{\II}}}(\emptyset)$ is simply the product of all $\msgt{\II}{i}(\emptyset)$ and $\msgt{\II}{i}(\emptyset)$
in its sub-tree, where the messages are towards the root. \RefEq{eq:genesis}
tells us that the global integral is not different.
Therefore, \refEqs{eq:recurse_1}{eq:recurse_2} we can completely expand the recursion
for the global integral.
For this, let $\uparrow i$
restrict the $\nb i$ to the factor that is higher than variable $i$ in the tree (\ie closer to the root $r$). Similarly let $\uparrow \II$ be the variable that is closer to the root than $\II$.
We can write the global integral as
 \begin{align}\label{eq:partition_by_message_partition}
   \qq(\emptyset) = \bigbptimes_{i, \II = \uparrow i} \msgt{i}{\II}(\emptyset) \bigbptimes_{\II, i = \uparrow \II} \msgt{\II}{i}(\emptyset)
 \end{align}

\RefProposition{th:simplify_sp} shows that these local integrals can be written in terms of local integrals of interest -- \ie
  \begin{align*}
    \msgt{\II}{i}(\emptyset) \quad = \quad \frac{\pht_\II(\emptyset)}{\msgbit{i}{\II}(\emptyset)} \quad \text{and} \quad
    \msgt{i}{\II}(\emptyset) \quad = \quad \frac{\pht_i(\emptyset)}{\msgbit{i}{\II}(\emptyset)}
  \end{align*}
Substituting from the equations above into \refEq{eq:partition_by_message_partition} we get
the equations of \refTheorem{th:integral_decompose}. 
\end{proof}




\subsection{The new factor-graph and semiring}\label{sec:newsemiring}


The decomposition of integral in \refTheorem{th:integral_decompose} means $\QQ(\msgss{\cdot}{\cdot})(\emptyset)$
has a factored form. Therefore, a factor-graph with 
$\msgss{\cdot}{\cdot}$ as the set of variables and 
three different types of factors corresponding to different
terms in the decomposition -- \ie $\PHT_\II(\msgss{\nb \II}{\II})(\emptyset)$, $\PHT_i(\msgss{\nb i \back \II}{i})(\emptyset)$ and $\MSGBIT{i}{\II}(\msg{i}{\II},\msg{\II}{i})(\emptyset)^{-1}$ -- can represent $\QQ(\msgss{\cdot}{\cdot})(\emptyset)$. 


\RefFigure{fig:spfg} shows a simple factor-graph 
and the corresponding SP factor-graph. The new factor-graph has one variable per each message in the
original factor-graph and three types of factors as discussed above.
Survey propagation is simply belief propagation applied to the  this new factor-graph using
the new semiring.
As before, BP messages are exchanged between variables and factors.
But here, we can simplify BP messages by substitution and only keep two types of factor-to-factor messages. 
We use $\msp{i}{\II}$ and $\msp{\II}{i}$ to denote these two types of SP messages.
These messages are exchanged between two types of factors, namely $\PHT_\II(\msgss{\nb \II}{\II})(\emptyset)$ and $\PHT_i(\msgss{\nb i \back \II}{i})(\emptyset)$.
Since the third type of factors $\MSGBIT{i}{\II}(\msg{i}{\II},\msg{\II}{i})(\emptyset)^{-1}$ 
are always connected to only two variables, $\msg{i}{\II}$ and $\msg{\II}{i}$, 
we can simplify their role in the SP message update to get 
\begin{align}
  \msp{i}{\II}(\msg{i}{\II}, \msg{\II}{i}) \; &\propto \; \bigspplus_{\back \msg{i}{\II},\msg{\II}{i}} \left ( \frac{\PHT_i(\msgss{\nb i}{i})(\emptyset)}{\MSGBIT{i}{\II}(\msg{i}{\II},\msg{\II}{i})(\emptyset)} \;
 \bigbptimes_{\JJ \in \nb i \back \II} \msp{\JJ}{i}(\msg{i}{\JJ},\msg{\JJ}{i}) \right ) \label{eq:spiI_first}\\
  \msp{\II}{i}(\msg{i}{\II}, \msg{\II}{i}) \; &\propto \; \bigspplus_{\back \msg{i}{\II},\msg{\II}{i} } \left (\frac{\PHT_\II(\msgss{\nb \II}{\II})(\emptyset)}{\MSGBIT{i}{\II}(\msg{i}{\II},\msg{\II}{i})(\emptyset)} \;
 \bigbptimes_{j \in \nb \II \back i} \msp{j}{\II}(\msg{j}{\II},\msg{\II}{j}) \right ) \label{eq:spIi_first}
\end{align}
where in all cases we are assuming the messages $\msgss{\cdot}{\cdot} \in \WW$ are consistent with each other -- \ie satisfy BP equations on the original factor-graph.
Note that, here again we are using the normalized BP message update and the normalization
factor is hidden using $\propto$ sign. This is possible because we assumed $\bptimes$ has an inverse.  
We can further simplify this update using the following proposition.

\begin{figure}
\centering
\includegraphics[width=1\textwidth]{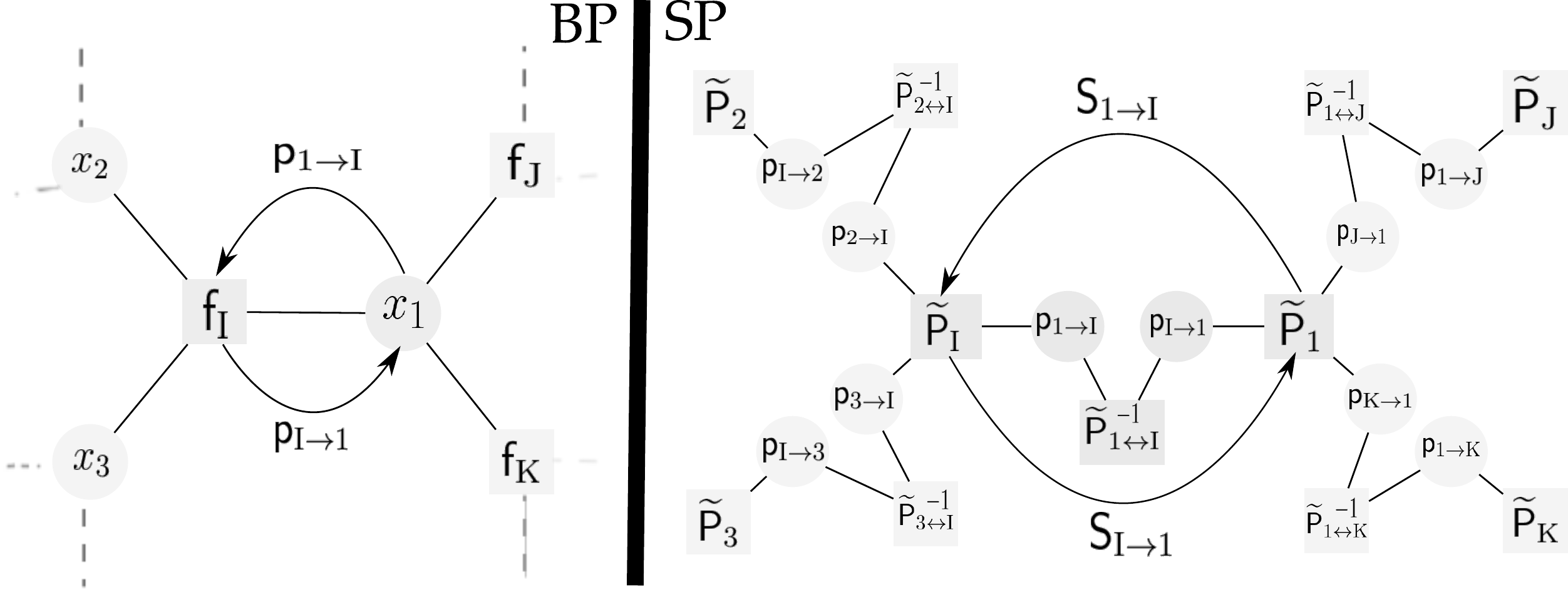}
\caption[Survey propagation factor-graph.]{Part of a factor-graph (left) and the corresponding SP factor-graph on the right.
The variables in SP factor-graph are the messages in the original graph. The SP factor-graph
has three type of factors: (I)~$\PHT_\II(.)(\emptyset)$, (II)~$\PHT_i(.)(\emptyset)$ and (III)~${\MSGBIT{i}{\II}(.)(\emptyset)}^{-1}$. 
As the arrows suggest, SP message updates
are simplified so that only two type of messages are exchanged: $\msp{i}{\II}$ and $\msp{\II}{i}$ between factors of type (I) and (II).
}
\label{fig:spfg}
\end{figure}

\begin{proposition} \label{th:simplify_sp} For $\msgss{\cdot}{\cdot} \in \WW$
  \begin{align} 
    \frac{\PHT_i(\msgss{\nb i}{i})(\emptyset)}{\MSGBIT{i}{\II}(\msg{\II}{i}, \msg{i}{\II})(\emptyset)} \; &= \; \MSGT{i}{\II}(\msgss{\nb i \back \II}{i})(\emptyset) \\
\quad &\text{and} \quad \notag \\
    \frac{\PHT_\II(\msgss{\nb \II}{\II})(\emptyset)}{\MSGBIT{i}{\II}(\msg{\II}{i}, \msg{i}{\II})(\emptyset)} \; &= \;  \MSGT{\II}{i}(\msgss{\nb \II \back i}{\II})(\emptyset)
  \end{align}
\end{proposition}
\begin{proof}
 By definition of $\pht_\II(\xs_\II)$ and $\msg{i}{\II}(\xx_i)$ in \refEqs{eq:mIi_partition}{eq:marg_factor_partition}
\begin{align*}
  \pht_\II(\xx_i) =  \msgt{\II}{i}(\xx_i) \bptimes \msg{i}{\II}(\xx_i) \quad &\Rightarrow \quad \bigbpplus_{\xx_i}\pht_\II(\xx_i) = \bigbpplus_{\xx_i} \msgt{\II}{i}(\xx_i) \bptimes \msg{i}{\II}(\xx_i) 
&\Rightarrow\\
\pht_{\II}(\emptyset) = \msgt{\II}{i}(\emptyset) \bptimes \left (\bigbpplus_{\xx_i} \msg{\II}{i}(\xx_i) \bptimes \msg{i}{\II}(\xx_i) \right ) \quad &\Rightarrow \quad \pht_{\II}(\emptyset) = \msgt{\II}{i}(\emptyset) \bptimes \msgbit{i}{\II}(\emptyset)
\end{align*}
where in the last step we used \refEq{eq:bi_direction}.

Similarly for the second statement of the proposition we have
\begin{align*}
  \pht_i(\xx_i) =  \msgt{i}{\II}(\xx_i) \bptimes \msg{\II}{i}(\xx_i) \quad &\Rightarrow \quad \bigbpplus_{\xx_i}\pht_i(\xx_i) = \bigbpplus_{\xx_i} \msgt{i}{\II}(\xx_i) \bptimes \msg{\II}{i}(\xx_i) 
&\Rightarrow\\
\pht_{i}(\emptyset) = \msgt{i}{\II}(\emptyset) \bptimes \big (\bigbpplus_{\xx_i} \msg{\II}{i}(\xx_i) \bptimes \msg{i}{\II}(\xx_i) \big ) \quad &\Rightarrow \quad \pht_{i}(\emptyset) = \msgt{i}{\II}(\emptyset) \bptimes \msgbit{i}{\II}(\emptyset)
\end{align*}
\end{proof}

The term on the l.h.s. in the proposition above appear in \refEqs{eq:spiI_first}{eq:spIi_first} and the terms on the r.h.s.~are local message integrals given by \refEqs{eq:mIi_partition}{eq:miI_partition}.
We can enforce $\msgss{\cdot}{\cdot} \in \WW$, by enforcing BP updates 
$\msg{i}{\II} = \MSG{i}{\II}(\msgss{\nb i \back \II}{i})$ and 
$\msg{\II}{i} = \MSG{\II}{i}(\msgss{\nb \II \back }{\II})$ ``locally'', 
during the message updates in the new factor-graph.
Combining this constraint with the simplification offered by
\refProposition{th:simplify_sp} gives us the SP message updates 
{\small
\begin{align}
  \msp{i}{\II}(\msg{i}{\II}) &\propto  \underset{\msgss{\nb i \back \II}{i}}{\bigspplus} \left (  
\ident\big ( \msg{i}{\II} = \MSG{i}{\II}(\msgss{\nb i \back \II}{i}) \big) \bptimes 
\MSGT{i}{\II}(\msgss{\nb i \back \II}{i})(\emptyset) \;  
 \bigbptimes_{\JJ \in \nb i \back \II} \msp{\JJ}{i}(\msg{\JJ}{i}) \right )\label{eq:spiI_semiring}\\
  \msp{\II}{i}(\msg{\II}{i})  &\propto  \underset{\msgss{\nb \II \back i}{\II}}{\bigspplus} \left ( 
\ident\big ( \msg{\II}{i} = \MSG{\II}{i}(\msgss{\nb \II \back i}{\II}) \big) \bptimes 
\MSGT{\II}{i}(\msgss{\nb \II \back i}{\II})(\emptyset) \;
 \bigbptimes_{j \in \nb \II \back i} \msp{j}{\II}(\msg{j}{\II}) \right ) \label{eq:spIi_semiring}
\end{align}
}%
where $\ident(.)$ is the identity function on the SP semiring, where $\ident(\truemath) = \identt{\bptimes}$ and
$\ident(\falsemath) = \identt{\spplus}$.

Here each SP message is a functional over all possible BP messages between the same variable and factor. However, in updating the SP messages, the identity functions
ensure that only the messages that locally satisfy BP equations are taken into account.
Another difference from the updates of \refEqs{eq:spiI_first}{eq:spIi_first} is that SP messages
have a single argument. This is because the new local integrals either depend on
$\msg{i}{\II}$ or $\msg{\II}{i}$, and not both.

\begin{example} In variational approach, survey propagation comes in two variations: entropic $\mathrm{SP}(\xi)$ and energetic $\mathrm{SP}(\mathbf{y})$ \cite{Mezard09}. For the readers familiar
with variational derivation of SP, here we express the relation to the algebraic approach.
According to the variational view, the partition function of the \textit{entropic SP}  is  $\sum_{\msgss{\cdot}{\cdot}} e^{\xi \log(\QQ(\msgss{\cdot}{\cdot})(\emptyset))}$, where $\QQ(\msgss{\cdot}{\cdot})(\emptyset)$ is the partition function for the sum-product semiring.
The entropic SP has an inverse temperature parameter, \aka \textit{Parisi parameter}, $\xi \in \Re$. It is easy to see that $\xi = 1$ corresponds to $\spplus = +, \bpplus=+$ and $\bptimes = \times$ in our algebraic approach. 
The limits of $\xi \to \infty$ corresponds to $\spplus = \max$. On the other hand, the limit of
$\xi \to 0$ amounts to ignoring $\QQ(\msgss{\cdot}{\cdot})(\emptyset)$ and corresponds to the counting SP; see \refSection{sec:uniformsp}.

The \textit{energetic $\mathrm{SP}(\mathbf{y})$} is different only in the sense that 
$\QQ(\msgss{\cdot}{\cdot})(\emptyset)$ in $\sum_{\msgss{\cdot}{\cdot}} e^{-\mathbf{y} \log(\QQ(\msgss{\cdot}{\cdot})(\emptyset))}$ is the ground state energy.
This corresponds to $\spplus = +, \bpplus=\max$ and $\bptimes = \sum$, and the limits of
the inverse temperature parameter $\mathbf{y} \to \infty$ is equivalent to $\spplus = \min, \bpplus=\min$ and $\bptimes = \sum$.
By taking an algebraic view we can choose between both operations and domains. For instance,
an implication of algebraic view is that all the variations of SP can be applied to the domain of complex numbers $\RR = \Ce$.
\end{example}

\subsection{The new integral and marginals}
Once again we can use \refTheorem{th:integral_decompose}, this time to approximate the \emph{SP integral}
$\QQ(\emptyset)(\emptyset) = \bigspplus_{\msgss{\cdot}{\cdot}} \QQ(\msgss{\cdot}{\cdot})(\emptyset)$
using local integral of SP messages.

The \emph{SP marginal} over each BP message $\msg{i}{\II}$ or $\msg{\II}{i}$ 
is the same as the corresponding SP message -- \ie $\PSP(\msg{i}{\II}) = \msp{i}{\II}(\msg{i}{\II})$.
To see this in the factor-graph of \refFigure{fig:spfg},
note that each message variable is connected to two factors, and both of these factors are already
contained in calculating one SP messages.

Moreover, from the SP marginals over messages we can recover the SP marginals over BP marginals
which we denote by $\PSP(\ph)(\xx_i)$.
For this, we simply need to enumerate all combinations of BP messages that produce a particular marginal (weighting them by their local integral $\PHT_i(\msgss{\nb i}{i})(\emptyset)$)
\begin{align}\label{eq:sp_marg}
  \PSP(\ph)(\xx_i) \; \propto \; \bigspplus_{\msgss{\nb i}{i}} \ident(\ph(\xx_i) = \PP(\msgss{\nb i}{i})(\xx_i)) \bptimes \PHT_i(\msgss{\nb i}{i})(\emptyset) \bigbptimes_{\II \in \nb i} \msp{\II}{i}(\msg{\II}{i}) 
\end{align}




\subsection{Counting survey propagation}\label{sec:uniformsp}
Previously we required the $\otimes$ operator to have an inverse, so that we can decompose the BP integral $\qq(\emptyset)$ into local integrals.
Moreover, for a consistent decomposition of the BP integral, SP and BP semiring previously shared the $\otimes$ operation.\footnote{This is because, if the expansion operation $\sptimes$ was different from the expansion operation of BP, $\bptimes$, the expanded form $\QQ(\msgss{\cdot}{\cdot})$ in the SP factor-graph
would not evaluate the integral $\qq(\emptyset)$ in the BP factor-graph, even in factor-graphs without any loops.}
 
Here, we lift these requirements by discarding the BP integrals altogether. This means SP semiring could be completely distinct from BP semiring and $(\RR, \otimes)$ does not have to be an Abelian group. This setting is particularly interesting when the SP semiring is sum-product over real domain 
\begin{align}
  \msp{i}{\II}(\msg{i}{\II}) \; &\propto \; \sum_{\msgss{\nb i \back \II}{i}}  
\ident\big ( \msg{i}{\II} = \MSG{i}{\II}(\msgss{\nb i \back \II}{i}) \big) 
 \prod_{\JJ \in \nb i \back \II} \msp{\JJ}{i}(\msg{\JJ}{i}) \label{eq:spiI_counting}\\
  \msp{\II}{i}(\msg{\II}{i}) \; &\propto \; \sum_{\msgss{\nb \II \back i}{\II}} 
\ident\big ( \msg{\II}{i} = \MSG{\II}{i}(\msgss{\nb \II \back i}{\II}) \big) 
 \prod_{j \in \nb \II \back i} \msp{j}{\II}(\msg{j}{\II})  \label{eq:spIi_counting}
\end{align}

Here, the resulting SP integral $\QQ(\msgss{\cdot}{\cdot}) = \sum_{\msgss{\cdot}{\cdot}} \ident(\msgss{\cdot}{\cdot} \in \WW)$ simply ``counts'' the number of BP fixed points and
SP marginals over BP marginals (given by \refEq{eq:sp_marg}) which approximates the frequency of a particular
marginal.
The original survey propagation equations in \cite{braunstein_survey_2002}, which are very
successful in solving satisfiability,
correspond to counting SP applied to the or-and semiring.
\begin{example}
Interestingly, in all min-max problems with discrete domains $\XX$, min-max BP messages can only take the values that are in the range of factors -- \ie $\RR = \YY$.
This is because any ordered set is closed under min and max operations.
Here, each counting SP message $\msp{i}{\II}(\msg{i}{\II}): \YY^{|\XX_i|} \to \Re$
is a discrete distribution over all possible min-max BP messages.
This means counting survey propagation where the BP semring is min-max is  
computationally ``tractable''.
In contrast, (counting) SP, when applied to sum-product BP over real domains, is not tractable.
This is because, in this case, each SP message is a distribution over an uncountable set:
$\msp{i}{\II}(\msg{i}{\II}): \Re^{|\XX_i|} \to \Re$.

In practice, (counting) SP is only interesting if it remains tractable. The most
well-known case corresponds to counting SP when applied to the or-and semiring. In this case
the factors are constraints and the domain of SP messages is $\{\truemath, \falsemath\}^{|\XX_i|}$.
Our algebraic perspective extends this set of tractable instances. For example, it show that 
counting SP can be used to count the number of fixed points of BP when applied to xor-and or min-max semiring.  
\end{example}

\section*{Conclusion}
This paper builds on previous work to addresses three basic questions about inference in graphical models:
(\textbf{I}) \emph{``what is an inference problem in a graphical model?''}  We use the combination of commutative semigroups and a factor-graph to answer this question in a broad sense, generalizing a variety of previous models.
(\textbf{II}) \emph{``How difficult is inference?''} By confining inference to four operations of min, max, sum and product that easily lend themselves to models of computation, we build an inference hierarchy
that is complete for \pspace\ and organizes inference problems into complexity classes with increasing levels of difficulty.
Only a few of these problems had previously been studied and 
only a handful of them now have a variational interpretation.
Moreover, we prove that inference for ``any'' commutative semiring is \NP-hard under randomized reduction, which generalizes previous results for particular semirings.  
(\textbf{III}) \emph{``When does distributive law help?''} After reviewing the algebraic form of belief propagation, and the conditions that allow its normalized form, we show that application of distributive law in performing exact inference is limited to inference problems with one marginalization operation. (\textbf{IV}) Finally we extend the algeberaic treatment of message passing techniques to survey propagation. This perspective generalizes SP to new settings and we hope that it also makes this powerfull message passing procedure -- that has its origins in statistical physics -- more accessible to the machine learning community.


\bibliographystyle{ACM-Reference-Format-Journals}
\bibliography{refs.bib}






\end{document}